\documentclass[11pt]{article} % For LaTeX2e
\usepackage{newtxtext}
\usepackage[left=1.25in, top=1in, bottom=1in, right=1.25in]{geometry}

\usepackage{amsmath,amsfonts,bm}

\def\1{\bm{1}}

\def\vg{{\bm{g}}}

\def\vm{{\bm{m}}}

\def\vv{{\bm{v}}}

\def\vx{{\bm{x}}}
\def\vy{{\bm{y}}}

\def\vDelta{{\bm{\Delta}}}

\DeclareMathAlphabet{\mathsfit}{\encodingdefault}{\sfdefault}{m}{sl}
\SetMathAlphabet{\mathsfit}{bold}{\encodingdefault}{\sfdefault}{bx}{n}

\DeclareMathOperator*{\argmax}{arg\,max}
\DeclareMathOperator*{\argmin}{arg\,min}

\newcommand{\norm}[1]{\left\|#1\right\|}
\def\abs#1{\left| #1 \right|}

\newcommand{\inner}[2]{\left\langle #1,#2 \right\rangle}

\newcommand{\eps}{\epsilon}

\newcommand{\adam}{$\mathtt{Adam}$\xspace}
\newcommand{\adamw}{$\mathtt{AdamW}$\xspace}
\newcommand{\gd}{$\mathtt{GD}$\xspace}
\newcommand{\sgd}{$\mathtt{SGD}$\xspace}
\newcommand{\signgd}{$\mathtt{SignGD}$\xspace}
\newcommand{\frankwolfe}{$\mathtt{Frank}$-$\mathtt{Wolfe}$\xspace}
\newcommand{\nsdwd}{$\mathtt{NSD}$-$\mathtt{WD}$\xspace}
\newcommand{\normalizedgd}{$\mathtt{Normalized}$-$\mathtt{GD}$\xspace}

\usepackage[colorlinks,
            linkcolor=red,
            anchorcolor=blue,
            citecolor=blue,
            backref=page,
            urlcolor=black
            ]{hyperref}       % hyperlinks
\usepackage{url}            % simple URL typesetting
\usepackage{nicefrac}       % compact symbols for 1/2, etc.
\usepackage[dvipsnames]{xcolor}  
\usepackage{amsfonts, amsmath, amssymb, amsthm, bm, mathtools}
\usepackage{cancel}
\usepackage{cleveref}
\usepackage{enumitem}
\usepackage{natbib}
\usepackage{thmtools} 
\usepackage{thm-restate}
\usepackage{caption}
\usepackage{subcaption}
\usepackage{nicefrac}
\usepackage{xfrac}
\usepackage{algorithms}
\usepackage{xspace}
\usepackage[disable]{todonotes}
\usepackage[T1]{fontenc}
\newtheorem{theorem}{Theorem}[section]
\newtheorem{definition}[theorem]{Definition}
\newtheorem{lemma}[theorem]{Lemma}

\crefdefaultlabelformat{#2#1#3}
\creflabelformat{equation}{#2#1#3}
\renewcommand{\vec}[1]{{\boldsymbol{\mathbf{#1}}}} 
\newcommand{\zhiyuan}[1]{{\color{cyan} [ZL: #1]}}
\newcommand{\shuo}[1]{{\color{LimeGreen} [SX: #1]}}
\renewcommand{\zhiyuan}[1]{}
\renewcommand{\shuo}[1]{}
\title{Implicit Bias of AdamW: $\ell_\infty$ Norm Constrained Optimization}
\date{}
\author{
Shuo Xie\qquad\qquad Zhiyuan Li \\
Toyota Technological Institute at Chicago\\
\texttt{\{shuox,zhiyuanli\}@ttic.edu} 
}

\begin{document}

\maketitle

\begin{abstract}
Adam with decoupled weight decay, also known as AdamW, is widely acclaimed for its superior performance in language modeling tasks, surpassing Adam with $\ell_2$ regularization in terms of generalization and optimization. However, this advantage is not theoretically well-understood. One challenge here is that though intuitively Adam with $\ell_2$ regularization optimizes the $\ell_2$ regularized loss, it is not clear if AdamW optimizes a specific objective. In this work, we make progress toward understanding the benefit of AdamW by showing that it implicitly performs constrained optimization. More concretely, we show in the full-batch setting, if AdamW converges with any non-increasing learning rate schedule whose partial sum diverges, it must converge to a KKT point of the original loss under the constraint that the $\ell_\infty$ norm of the parameter is bounded by the inverse of the weight decay factor. This result is built on the observation that Adam can be viewed as a smoothed version of SignGD, which is the normalized steepest descent with respect to $\ell_\infty$ norm, and a surprising connection between normalized steepest descent with weight decay and Frank-Wolfe.
\end{abstract}
\section{Introduction}\label{sec:intro}
\adam~\citep{kingma2014adam} and its variant \adamw~\citep{loshchilov2017decoupled} have been the most successful and widely used optimization algorithms in deep learning, especially for large language models (LLMs), whose pre-training costs massively and cannot be done with SGD. Despite its tremendous empirical success, we lack a good theoretical understanding of \adam's underlying work and the roles of its hyperparameters, in particular, \emph{weight decay}. \adamw achieves better optimization and generalization over \adam and decouples the effect of the learning rate and the weight decay coefficient by using a different implementation of weight decay~\citep{loshchilov2017decoupled}.\shuo{Shall we say decouple first and then achieve better performance?}  While \adam implements weight decay as a $\ell_2$ regularizer of the training objective, \adamw directly shrinks its weight per step, known as the \emph{decoupled weight decay}~(see \Cref{alg:adamw}).

\begin{algorithm}
    \caption{\colorbox{Thistle}{Adam with $\ell_2$ regularization} and \colorbox{SpringGreen}{Adam with decoupled weight decay (\adamw)}}\label{alg:adamw}
    \begin{algorithmic}
        \INPUT{$\beta_1, \beta_2>0$, initialization $\vx_0$, total steps $T$, learning rate schedule $\{\eta_t\}_{t=1}^T$, weight decay coefficient $\lambda$}
        \State $\vm_0 \gets \vec{0}, \vv_0 \gets \vec{0}$  
        \For{$t=1, 2, \cdots, T$}
        % \State $t \gets t+1$
        \State $\vg_t \gets \nabla L(\vx_{t-1})$\colorbox{Thistle}{$+\lambda \vx_{t-1}$}
        \State $\vm_t \gets \beta_1 \vm_{t-1} + (1-\beta_1) \vg_t$
        \State $\vv_t \gets \beta_2 \vv_{t-1} + (1-\beta_2) \vg_t^2$
        \State $\vx_t \gets \vx_{t-1} -\eta_t \frac{\vm_t}{\sqrt{\vv_t}}$\colorbox{SpringGreen}{$-\lambda \eta_t \vx_{t-1}$}
        \EndFor
        \State \Return $\vx_T$
    \end{algorithmic}
\end{algorithm}

However, the advantage of \adamw over \adam is mostly empirical and our theoretical understanding is quite limited. \citet{zhuang2022understanding} argues that one desirable property that \adamw has while \adam does not is scale-freeness, meaning \adamw yields the same optimization trajectory if loss is multiplied by any positive constant. Yet this property does not give us enough information to understand the difference regarding the optimization processes and the final learned solutions between \adamw and \adam with $\ell_2$ regularization. Intuitively, if \adam with $\ell_2$ regularization converges to some point, it converges to at least a stationary point of the regularized loss function, if not a minimizer. But for \adamw, it is even not clear if it is optimizing any (regularized) loss function. Thus, towards taking the first step of understanding the benefit of decoupled weight decay in \adamw, we ask the following question:

\begin{quote}
    Which solution does \adamw converge to, if it converges?
\end{quote}

Our following main result \Cref{thm:main} characterizes the implicit bias of \adamw in the \emph{deterministic} case, where a full-batch loss is used:
\begin{theorem}\label{thm:main}
    For any continuously differentiable function $L:\mathbb{R}^d\to\mathbb{R}$,  $\beta_1 \leq \beta_2<1 $,  initialization $\vx_0$ 
 and non-increasing learning rate $\{\eta_t\}_{t=1}^\infty$ such that $\sum_{t=1}^\infty \eta_t = \infty$, if the iterates of \adamw $\{\vx_t\}_{t=0}^\infty$ on $L$ converges to some $\vx_\infty$, then $\vx_\infty$ is a KKT point (\Cref{defi:KKT_points}) of the constrained optimization problem $\min_{\norm{\vx}_\infty \leq \frac{1}{\lambda}} L(\vx)$. 
 
If $L$ is additionally convex, then
\adamw converges to the constrained minimizer, \emph{i.e.}, $\vx_\infty\in \argmin_{\norm{\vx}_\infty \leq \frac{1}{\lambda}} L(\vx)$.
\end{theorem}

Despite being simplistic, the full-batch setting is still a very interesting and highly non-trivial regime, because the two main hypotheses of why \adam outperforms \sgd got challenged recently in the deterministic regime~\citep{kunstner2023noise}. The first hypothesis is that \adam outperforms \sgd by better handling heavy-tailed noise~\citep{zhang2020adaptive}. However, \citet{kunstner2023noise} finds that  \adam still outperforms \gd for optimizing language tasks even in the full-batch setting. The second hypothesis is the smoothness of the training loss landscape can linearly increase as the gradient norm increases and thus clipping or normalization is necessary for gradient descent. Intriguingly,  \citet{kunstner2023noise} finds that normalizing each update of GD cannot close the gap towards \adam in the full-batch setting, but normalizing \emph{each coordinate} to its sign (\emph{i.e.}, \signgd) closes the gap. 

The theoretical results and analysis in this work support the empirical observation made by \citet{kunstner2023noise}. The way we prove \Cref{thm:main} is first to prove that normalized steepest descent with weight decay (\nsdwd) for any norm $\norm{\cdot}$ must converge to KKT points of the constrained optimization problem $\min_{\norm{\vx}\le \frac{1}{\lambda}} L(\vx)$ (\Cref{thm:nsd_wd_main}). Then we show that \adamw asymptotically behaves just like \signgd with weight decay, which is the normalized steepest descent w.r.t. $\ell_\infty$ norm with weight decay and the same proof framework generalizes to \adamw if $\beta_1\le \beta_2<1$. The condition $\beta_1\le \beta_2$ is crucial, and we provide a counter-example where $1>\beta_1>\beta_2$ and \adamw converges somewhere else, instead of a constrained minimizer.

It remains interesting why \signgd outperforms \normalizedgd in the full batch setting. They are both normalized steepest descent but with respect to different norms -- \normalizedgd picks the steepest direction under the geometry of $\ell_2$ norm, while \signgd picks the steepest direction under the geometry of the $\ell_\infty$ norm.  It is natural to make the following conjecture: \emph{\adam outperforms \gd due to its utilization of $\ell_\infty$ geometry, under which the loss function could have better properties, e.g., smaller smoothness}. Our main result~\Cref{thm:main} provides positive evidence for this conjecture. We also provide a convergence analysis for normalized steepest descent with weight decay for convex loss, where the suboptimality against the constrained minimizer in norm ball of radius $\frac{1}{\lambda}$ vanishes is $O(\frac{H}{T\lambda^2})$, where $T$ is the total number of steps, $\lambda$ is the weight decay factor, and $H$ is the smoothness of loss w.r.t. the particular norm used for picking the steepest descent direction. Based on the convergence bound, we construct a concrete $d$-dimensional loss function in \Cref{subsec:convex_convergence} whose minimizer $\vx^*$ satisfies $\norm{\vx^*}_2\approx \sqrt{d} \norm{\vx^*}_\infty$ and \signgd with weight decay converges much faster than \normalizedgd with weight decay because \signgd with weight decay can use a $\sqrt{d}$ times larger weight decay factor $\lambda$ than \normalizedgd. 

\paragraph{Contributions.} Below we summarize our contributions:
\begin{enumerate}
    \item In \Cref{subsec:convex_convergence}, we prove normalized steepest descent with weight decay optimizes convex functions under norm constraints~(\Cref{thm:1/t_convergence}). In \Cref{subsec:convergence_to_KKT_points}, we prove it must converge to KKT points of the norm-constrained optimization problem for general loss functions if it converges with a learning rate schedule whose partial sum diverges~(\Cref{thm:nsd_wd_main}). 
    \item In \Cref{sec:adamw}, we prove \adamw must converge to KKT points of the norm-constrained optimization problem for general loss functions if it converges with a non-increasing learning rate schedule whose partial sum diverges~(\Cref{thm:main}). 
    \item Towards generalizing the proof of \Cref{thm:nsd_wd_main} to \Cref{thm:main}, we prove a novel and tight upper bound on average update size of \adam~(\Cref{lem:amortized_update_bound}), which holds even for non-deterministic settings as well and might be of independent interest to the community. We test various predictions made by our bound in experiments.
\end{enumerate}
\section{Preliminaries and Notations}\label{sec:prelim}

\paragraph{Notations:} We use $\norm{\cdot}$ to denote a general norm and $\norm{\cdot}_*$ to denote its dual norm. We say a function $L:\mathbb{R}^d\to \mathbb{R}$ has $H$-lipschitz gradient w.r.t. norm $\norm{\cdot}$ for some $H>0$ iff for all $\vx,\vy\in\mathbb{R}^d$, $\norm{\nabla L(\vx) - \nabla L(\vy)}_* \leq H \norm{\vx-\vy}$. We define the \emph{smoothness} of loss $L$ as the smallest positive $H$ w.r.t. $\norm{\cdot}$ such that $L$ has $H$-lipschitz gradient. We say a function $L:\mathbb{R}^d\to \mathbb{R}$ is convex iff for any $\vx,\vy\in\mathbb{R}^d$ and $\theta\in[0,1]$, it holds that $L(\theta\vx+(1-\theta)\vy)\le \theta L(\vx) + (1-\theta)L(\vy)$. We define the \emph{subgradients} of convex function $L$ at point $\vx$ as $\{\vg\in\mathbb{R}^d\mid L(\vy)\ge L(\vx)+\inner{\vy-\vx}{\vg},\forall \vy\in\mathbb{R}^d\}$, which is denoted by $\partial L(\vx)$. When $L$ is differentiable at $\vx$,  $\partial L(\vx)$ contains only one element, which is the gradient $\nabla L(\vx)$. In particular, all norms are convex functions and we have the following standard lemma for the subgradients of norms:

\begin{lemma}\label{lem:norm_subgradient}
    For any norm $\norm{\cdot}$ and $\vx\in\mathbb{R}^d$,  $\partial \norm{\vx} = \{\vDelta\in\mathbb{R}^d\mid \norm{\vDelta}_*=1,\inner{\vDelta}{\vx} =\norm{\vx}\}$.
\end{lemma}

\paragraph{Steepest Descent:} We say $\vv$ is a \emph{steepest descent direction} for objective function $L$ at current iterate $\vx$ w.r.t. norm $\norm{\cdot}$ iff $\norm{\vv} =1$ and $\inner{\vv}{\nabla L(\vx)} = \min_{\norm{\vv'}\le 1} \inner{\vv'}{\nabla L(\vx)}$. Thus for all steepest descent direction $\vv$, we have that $\inner{\vv}{\nabla L(\vx)} = -\norm{\nabla L(\vx)}_*$. 

Given initialization $\vx_0$, learning rate schedule $\{\eta_t\}_{t=0}^\infty$ and weight decay factor $\lambda$,
the $t$th iterate of \emph{normalized steepest descent} w.r.t. $\norm{\cdot}$ with decoupled weight decay is defined as 
\begin{equation}\label{eq:def_nsd_wd}
    \vx_{t} = (1-\lambda \eta_t)\vx_{t-1} - \eta_t \vDelta_t  \text{, where } \vDelta_t \in \argmax_{\norm{\vDelta} \leq 1} \nabla L(\vx_{t-1})^\top \vDelta.
\end{equation} 
Because the dual norm of the dual norm is always equal to the original norm, by \Cref{lem:norm_subgradient}, we can also characterize the steepest descent directions as the subgradient of its dual norm.
\begin{lemma}\label{lem:subgradient_of_dual_norm_is_steepest_descent_direction}
    $\argmax\limits_{\norm{\vDelta} \leq 1} \nabla L(\vx)^\top \vDelta = \left. \partial \norm{\vy}_*\right|_{\vy = \nabla L(\vx)}$.
\end{lemma}
For completeness, we also define the \emph{steepest descent} w.r.t. $\norm{\cdot}$ with decoupled weight decay below, though we will not use it in our analysis. If we pick $\ell_2$ norm, \Cref{eq:def_sd_wd} becomes standard gradient descent.  
\begin{equation}\label{eq:def_sd_wd}
    \tilde \vx_{t} = (1-\lambda \eta_t)\tilde \vx_{t-1} - \eta_t \tilde \vDelta_t  \text{, where } \tilde \vDelta_t \in \argmax_{\tilde \vDelta\in\mathbb{R}^d} \left(\nabla L(\vx_{t-1})^\top \tilde \vDelta - \frac{1}{2}\norm{\tilde \vDelta}^2\right).
\end{equation}

It can be shown that for each steepest descent update $\tilde \vDelta$ for objective $L$ at $\vx$, there exists some normalized steepest descent update $\vDelta$ satisfying $\tilde \vDelta = \norm{\nabla L(\vx)}_*\vDelta$.

\section{Warm Up: Implicit Bias of Normalized Steepest Descent w. Weight Decay}\label{sec:warm_up}

In this section, we aim to present some high-level intuition about the constrained-minimization implicit bias of \adamw~(\Cref{thm:main}), by showing the same implicit bias for \signgd with weight decay, or equivalently, normalized steepest descent w.r.t. $\ell_\infty$ norm. \adamw is arguably a smoothed version of \signgd, which reduces the correlation between its numerator and denominator by using past moving average and thus reduces the biasedness of the update direction in the presence of noise. But intuitively, their behaviors are similar when there is no noise and the learning rate is small. 

Our analysis in this section holds for all norms, including the non-differentiable ones, like $\norm{\cdot}_\infty$.
\zhiyuan{more comments}

\subsection{Convex setting: constrained optimization}\label{subsec:convex_convergence} %or just "convex case"

In this subsection, we give a simple non-asymptotic convergence analysis for normalized Steepest descent w. weight decay (\nsdwd) w.r.t. to general norms over smooth convex loss functions. If the norm of initialization is no larger than $\frac{1}{\lambda}$ where $\lambda$ is the weight decay factor then surprisingly \nsdwd is exactly equivalent to a well-known optimization algorithm in literature, \frankwolfe~\citep{frank1956algorithm}, where the constraint set here is the norm ball with radius $\frac{1}{\lambda}$.  If the norm of initialization is larger than $\frac{1}{\lambda}$, then the analysis contains an additional phase where the norm of iterates linearly converges to $\frac{1}{\lambda}$. In this case, the iterate of \nsdwd may always be outside the $\frac{1}{\lambda}$ norm ball, but still, the convergence analysis of \frankwolfe can be adopted (\emph{e.g.},~\citet{jaggi2013revisiting}). 
First, we show that the norm of the iterates will shrink to $\frac{1}{\lambda}$ as long as the norm of each update is bounded by $1$, i.e., $\norm{\vDelta_t} \leq 1$. Note this conclusion doesn't use the convexity of the function $L(\vx)$ nor the update $\vDelta_t$ being the steepest descent direction. It can hold under non-deterministic settings.
\begin{lemma} \label{lem:convergence_to_ball}
    For any learning rate schedule $\{\eta_t\}_{t=1}^\infty$ and update $\{\vDelta_t\}_{t=1}^\infty$ such that $\lambda \eta_t < 1$ and $\norm{\vDelta_t} \leq 1$, $\norm{\vx_t} - \frac{1}{\lambda} \le \max\left( e^{-\lambda \sum_{i=1}^{t} \eta_i } \left(\norm{\vx_0} - \frac{1}{\lambda} \right) ,0\right)$.
\end{lemma}
The proof is deferred to \Cref{sec:convergence_to_ball}. \Cref{lem:convergence_to_ball} shows that $\vx_t$ is either always inside the norm ball with radius $\frac{1}{\lambda}$, or their distance shrinks exponentially as the sum of learning rates increases. Whenever $\vx_t$ gets into the norm ball with radius $\frac{1}{\lambda}$, $\vx_t$ will not leave it and the remaining trajectory of \nsdwd is exactly the same as \frankwolfe, as shown in the following theorem. We note the relationship between \frankwolfe and steepest descent algorithms is also observed very recently in the continuous case~\citep{chen2023lion}.

\begin{theorem}
    For any norm $\norm{\cdot}$,  weight decay $\lambda$, and $\norm{\vx_{t-1}}\le \frac{1}{\lambda}$, \nsdwd with learning rate $\eta_t<\frac{1}{\lambda}$ and \frankwolfe~(\Cref{alg:frank_wolfe}) with step size $\gamma_t = \eta_t\lambda$ and convex set $\mathcal{X} \triangleq \{\vy\mid \norm{\vy}\le \frac{1}{\lambda}\}$ generate the same next iterate $\vx_{t}$.
\end{theorem}

\begin{algorithm}
    \caption{\frankwolfe}\label{alg:frank_wolfe}
    \begin{algorithmic}
        \INPUT{convex set $\mathcal{X}$, $\vx_0 \in \mathcal{X}$, total steps $T$, step sizes $\{\gamma_t\}_{t=1}^T$}
        % \State $\vm_0 \gets \vec{0}, \vv_0 \gets \vec{0}$  
        \For{$t=1, 2, \cdots, T$}
        \State $\vy_t \gets \argmin_{\vy \in \mathcal{X}} \nabla L(\vx_{t-1})^\top \vy$
        \State $\vx_t \gets (1-\gamma_t) \vx_{t-1} + \gamma_t \vy_t$
        % \State $\vg_t \gets \nabla L(\vx_{t-1})$
        % \State $\vm_t \gets \beta_1 \vm_{t-1} + (1-\beta_1) \vg_t$
        % \State $\vv_t \gets \beta_2 \vv_{t-1} + (1-\beta_2) \vg_t^2$
        % \State $\vx_t \gets \vx_{t-1} -\eta_t \frac{\vm_t}{\sqrt{\vv_t}} -\lambda \eta_t \vx_{t-1}$
        \EndFor
        \State \Return $\vx_T$
    \end{algorithmic}
\end{algorithm}

Define $\vx^*=\argmin_{\norm{\vx}\leq \frac{1}{\lambda}} L(\vx)$ to be the constrained minimizer of convex function $L(\vx)$. We first compute how much the gap between $L(\vx_t)$ and $L(\vx^*)$ can decrease in one normalized steepest descent step when the iterate $\vx_t$ is bounded. 
\begin{lemma}[Descent Lemma for Smooth Convex Loss]\label{lem:one_step_nsd}
    Suppose loss function $L$ is convex and has $H$-lipschitz gradient w.r.t. norm $\norm{\cdot}$. For iterates $\{\vx_t\}$ in \nsdwd (\Cref{eq:def_nsd_wd}), we have that \[L(\vx_{t}) - L(\vx^*) \leq (1-\lambda \eta_t)(L(\vx_{t-1}) - L(\vx^*)) + \frac{H}{2} \eta_t^2\left(1 + \lambda \norm{\vx_{t-1}} \right)^2.\] 
\end{lemma}

The proof of \Cref{lem:one_step_nsd} is deferred to \Cref{sec:convergence_nsd}. 
With \Cref{lem:one_step_nsd}, we can prove the convergence of $L(\vx_t)$ for learning rate schedules with certain conditions. The proof is also deferred to \Cref{sec:convergence_nsd}.
\begin{theorem}\label{thm:any_rate_convergence}
Assume that $\eta_t \geq 0$, $\lim_{t \rightarrow \infty} \eta_t =0$ and $\sum_{t=1}^\infty \eta_t = \infty$. For any convex loss $L$ with $H$-lipschitz gradient, $\lim_{t \rightarrow \infty} L(\vx_t) = L(\vx^*)$.
\end{theorem}

We also provide a specific example of learning rates $\{\eta_t\}_{t=1}^\infty$ that can achieve $O(\frac{1}{t})$ convergence of $f(\vx_t)$, which is the same as \frankwolfe over convex objectives~\citep{jaggi2013revisiting} and the proof is standard. For completeness, we provide a proof of \Cref{thm:1/t_convergence} in \Cref{sec:convergence_nsd}.

\begin{theorem}\label{thm:1/t_convergence}
     Define $B = \max{\{\norm{\vx_0}, \frac{1}{\lambda}\}}$. For \nsdwd with learning rate schedule $\eta_t = \frac{2}{\lambda (t+1)}$, we have $L(\vx_t) - L(\vx^*) \leq \frac{2H(1+\lambda B)^2}{(t+2)\lambda^2}$ for $t \geq 1$.
\end{theorem}
 Note that the descent rate highly depends on the smoothness coefficient $H$, which is determined by the selected norm. 
Therefore, we provide a synthetic example that may demonstrate the advantage of $\ell_\infty$ norm as mentioned in \Cref{sec:intro}. For some constant $\vx^* \in \mathbb{R}^d$, the loss function $g:\mathbb{R}^{d} \rightarrow \mathbb{R}$ is defined as 
\begin{equation}\label{eq:synthetic_example}
    g(\vx) = \sum_{i=1}^{d} \frac{(\vx_i-\vx^*_i)^2}{i^2}.
\end{equation}
The Hessian matrix $\nabla^2 g$ is a diagonal matrix with diagonal entries $\{\frac{2}{i^2}\}_{i=1}^d$. For $\ell_2$ norm, the smoothness coefficient is the largest eigenvalue of Hessian matrix, which is $2$. For $\ell_\infty$ norm, the smoothness coefficient is the sum of the diagonal entries because of the diagonality, which is $\sum_{i=1}^d \frac{2}{i^2}$. It is upper bounded by $\sum_{i=1}^\infty \frac{2}{i^2} = \frac{\pi^2}{3}$ for all dimension $d$. However, if $\vx^*$ is set to be in the unit $\ell_\infty$ norm ball, its $\ell_2$ norm can be as large as $\sqrt{d}$, which makes the suboptimality bound in \Cref{thm:1/t_convergence} for $\ell_2$ norm normalized steepest descent with weight decay $d$ times larger than its $\ell_\infty$ counterpart. We implement steepest descent with $\ell_\infty$ norm and $\ell_2$ norm on a $100$-dimension example in \Cref{sec:synthetic_exp} and find that $\ell_\infty$ norm can indeed work better. 

\subsection{Non-convex setting: convergence to KKT points}\label{subsec:convergence_to_KKT_points}

In this subsection, we study the implicit bias of \signgd (or more generally, \nsdwd) when the loss is non-convex. In such case, last-iterate parameter convergence is, in general, difficult to show\footnote{Indeed, even for convex case, \frankwolfe may not converge in parameter.\citep{bolte2023iterates}}, and thus we turn to study \emph{what parameters \signgd and \nsdwd can converge to}. Our main results \Cref{thm:nsd_wd_main} show that such parameters must be the KKT points~(see \Cref{defi:KKT_points}) of the constrained optimization problems. In particular, if the objective is convex, since the norm ball constraint is always convex for all norm, all KKT points are constrained minimizers.

\begin{definition}[KKT points]\label{defi:KKT_points}
    We say $\vx^*\in\mathbb{R}^d$ is a KKT point of the constrained optimization problem $\min_{\norm{\vx}\le \frac{1}{\lambda}} L(\vx)$ iff there exists $s^*\ge0$ such that $\norm{\vx^*}\le \frac{1}{\lambda}$, $0\in \nabla L(\vx^*)+ \partial (s^*\norm{\vx})\vert_{\vx=\vx^*}$, $s^*(\norm{\vx^*}-\frac{1}{\lambda})=0$.
\end{definition}

For convex $L$, all KKT points $\vx^*$ are optimal and the dual variable $s^*\ge0$ is the certificate for the optimality. To see that, for any other $\norm{\vy}\le \frac{1}{\lambda}$, it holds that $L(\vy)\ge L(\vy)+s^*(\norm{\vy}-\frac{1}{\lambda})\ge L(\vx^*)+s^*(\norm{\vx^*}-\frac{1}{\lambda})$, where the second inequality is because $L(\vx)+s^*\norm{\vx}$ is also convex and $0$ is its subgradient at $\vx^*$. Thus we conclude $L(\vy)\ge  L(\vx^*)+s^*(\norm{\vx^*}-\frac{1}{\lambda}) = L(\vx^*)$.

Now we state the main result for this subsection.
\begin{theorem}[Non-convex, KKT]\label{thm:nsd_wd_main}
    For any continuously differentiable function $L:\mathbb{R}^d\to\mathbb{R}$,  initialization $\vx_0$,
 and learning rate schedule $\{\eta_t\}_{t=1}^\infty$ such that $\sum_{t=1}^\infty \eta_t = \infty$, if the iterates of \nsdwd $\{\vx_t\}_{t=0}^\infty$ on $L$ converges to some $\vx_\infty$, then $\vx_\infty$ is a KKT point (\Cref{defi:KKT_points}) of the constrained optimization problem $\min_{\norm{\vx} \leq \frac{1}{\lambda}} L(\vx)$. 
\end{theorem}

To prove \Cref{thm:nsd_wd_main}, we use the following alternative characterization for KKT points of $\min_{\norm{\vx}\le \frac{1}{\lambda}} L(\vx)$ below based on \Cref{lem:norm_subgradient}.
\begin{lemma}\label{lem:alternative_KKT}
    $\vx$ is a KKT point of $\min_{\norm{\vx}\le \frac{1}{\lambda}} L(\vx)$ iff $\norm{\vx}\le\frac{1}{\lambda}$ and $\inner{-\lambda \vx}{\nabla L(\vx)} = \norm{\nabla L(\vx)}_*$.
\end{lemma}

With \Cref{lem:alternative_KKT}, next we illustrate the intuition for \Cref{thm:nsd_wd_main} for the case where the dual norm $\norm{\cdot}_*$ is continuously differentiable at $\nabla L(\vx_\infty)$ (like $\ell_2$-norm and $\nabla L(\vx_\infty)\neq \mathbf{0}$). For sufficiently large $t$, when $\nabla L(\vx_t)$ gets sufficiently close to $\nabla L(\vx_\infty)$, the descent direction $-\vDelta_t $ is unique, equal to $-\nabla \norm{\nabla L(\vx_t)}_*$ by \Cref{lem:subgradient_of_dual_norm_is_steepest_descent_direction}, and satisfies $\inner{\vDelta_t}{\nabla L(\vx_t)} = \norm{\nabla L(\vx_t)}_*$. Taking $t\to\infty$ we get $\inner{-\nabla \norm{\nabla L(\vx_\infty)}_*}{\nabla L(\vx_\infty)} = \norm{\nabla L(\vx_\infty)}_*$. Moreover, we must have $ \nabla \norm{\nabla L(\vx_\infty)}_*+ \lambda\vx_\infty = \lim_{t\to\infty} \left( \nabla \norm{\nabla L(\vx_t)}_*+ \lambda\vx_t\right)=0$, otherwise  $\vx_t$ keeps moving towards $\lim_{t\to\infty} \left( \nabla \norm{\nabla L(\vx_t)}_*+ \lambda\vx_t\right)$ since $\sum_{t=1}^\infty \eta_t =\infty$ and thus $\vx_\infty$ cannot be the limit. This implies the second condition in \Cref{lem:alternative_KKT}. The first condition that $\norm{\vx_\infty}\le\frac{1}{\lambda}$ is immediate from \Cref{lem:convergence_to_ball} and that $\sum_{t=1}^\infty \eta_t =\infty$. 

However, the above intuition no longer works when dual norm $\norm{\cdot}_*$ is not differentiable at $\nabla L(\vx_\infty)$. This could happen for $\ell_\infty$ norm where the dual norm is $\ell_1$ norm and $\nabla L(\vx_\infty)$ with coordinates of value $0$, because the subgradient of absolute value function at $0$, $\partial |0|$, could be anything between $-1$ and $1$. And more generally, this could happen for any norm and $\nabla L(\vx_\infty)=\mathbf{0}$. If the limit point $\vx_\infty$ has zero gradient for $L$, then the steepest descent direction $-\vDelta$ is provably not continuous around $\vx_\infty$. 

The following lemma (\Cref{lem:nsd_wd_KKT}) circumvents the above issue by considering the weighted average of past steepest descent directions, which provably converges, given the iterates $\{\vx_t\}_{t=1}^\infty$  converge.  \Cref{thm:nsd_wd_main} is a direct combination of \Cref{lem:nsd_wd_KKT} and \Cref{lem:alternative_KKT} and we omit its proof. The proof of \Cref{lem:nsd_wd_KKT} is deferred into \Cref{sec:proof_nsd_wd_KKT}.

\begin{lemma}\label{lem:nsd_wd_KKT}
For any learning rate schedule $\{\eta_t\}_{t=1}^\infty$ satisfying $\sum_{t=1}^\infty \eta_t = \infty$, if the iterates of \nsdwd $\{\vx_t \}_{t=0}^\infty$ converges to some $\vx_\infty$, we have that
    \begin{enumerate}
        \item $\vDelta_\infty := \lim\limits_{T \rightarrow \infty} \frac{\sum_{t=1}^T \eta_t \vDelta_t}{\sum_{t=1}^T \eta_t}$ exists and $\vDelta_\infty= -\lambda \vx_\infty$.
        \item $\inner{\nabla L(\vx_\infty)}{\vDelta_\infty} = \norm{\nabla L(\vx_\infty)}_*$.
        \item $\norm{\vDelta_\infty} \leq 1$.
    \end{enumerate}
\end{lemma}

\zhiyuan{talk about extension to normalized steepest descent with momentum if time permits, also Lion} 
\section{Implicit Bias of AdamW}\label{sec:adamw}
In this section, we extend the analysis on \nsdwd in \Cref{sec:warm_up} to \adamw to prove that the converged parameters of \adamw is the KKT point of the constrained optimization problem. The proof relies on an upper bound of average update size of \adamw and we find that the bound can also be used to guide hyperparameter tuning in empirical study.

We first state the analog of \Cref{lem:nsd_wd_KKT} for \adamw with the norm being $\ell_\infty$ norm since we treat \adamw as a smoothed version of \signgd, which is \Cref{lem:property_converged_point}. Here we additionally assume that $\{ \eta_t\}_{t=1}^\infty$ is non-increasing and $\vDelta_t$ is defined as $\frac{\vm_t}{\sqrt{\vv_t}}$ from \Cref{alg:adamw}. \Cref{thm:main} is again a direct combination of \Cref{lem:alternative_KKT} and \Cref{lem:property_converged_point}.
\begin{lemma}\label{lem:property_converged_point} 
    For non-increasing learning rate schedule $\{\eta_t\}_{t=0}^\infty$ satisfying $\sum_{t=1}^\infty \eta_t = \infty$ and $\beta_2 \geq \beta_1$, we get $\{\vx_t \}_{t=1}^\infty$ by running AdamW with weight decay factor $\lambda$. If $\{\vx_t \}_{t=0}^\infty$ converges to some $\vx_\infty$, then it holds that
    \begin{enumerate}
        \item $\vDelta_\infty := \lim\limits_{T \rightarrow \infty} \frac{\sum_{t=1}^T \eta_t \vDelta_t}{\sum_{t=1}^T \eta_t}$ exists and $\vDelta_\infty= -\lambda \vx_\infty$.
        \item $\inner{\nabla L(\vx_\infty)}{\vDelta_\infty} = \norm{\nabla L(\vx_\infty)}_1$.
        \item $\norm{\vDelta_\infty}_\infty \leq 1$.
    \end{enumerate}
\end{lemma}
The condition $\beta_1 \leq \beta_2$ is necessary for the conclusion to hold. Otherwise, the iterates can converge outside the $\ell_\infty$ norm ball with radius $\frac{1}{\lambda}$ as shown in \Cref{sec:counter_example}. 

The first two properties in \Cref{lem:property_converged_point} follow from a similar argument for \Cref{lem:property_converged_point}, and the main technical difficulty here lies in the proof of the third property. This is because for any single $t$, $\norm{\vDelta_t}$ could be larger than $1$, which is different from the case of \nsdwd. To prove the third property, we need a tight upper bound for the average update size of \adam-like update rule, which is \Cref{lem:amortized_update_bound}. The proof of \Cref{lem:property_converged_point} is deferred to \Cref{sec:adamw_details}. 
\subsection{Upper bound for average update size of \adam}\label{sec:amortized_update_bound}
As mentioned earlier, \adam updates $\norm{\frac{\vm_t}{\sqrt{\vv_t}}}$ can easily go beyond $1$ and thus we prove the following upper bound for the average update size of \adam~(\Cref{lem:amortized_update_bound}). The proof of \Cref{lem:amortized_update_bound} is deferred to \Cref{sec:proof_amortized_update_bound}. 
\begin{lemma}\label{lem:amortized_update_bound}
    Given any $\beta_1 \leq \beta_2<1$, suppose scalar sequences $\{v_t\}_{t=0}^\infty$ and $\{g_t\}_{t=1}^\infty$ satisfy that $v_0 \geq 0, v_1 >0$ and $v_t-\beta_2 v_{t-1} \geq (1-\beta_2)g_t^2$ for $t \geq 1$. Given initial value $|m_0| \leq \sqrt{v_0}$, define $m_t = \beta_1 m_{t-1} + (1-\beta_1) g_t$ and $\Delta_t = \frac{m_t}{\sqrt{v_t}}$ for $t \geq 1$. For any coefficients $\{\eta_t\}_{t=1}^\infty$ and $T\in\mathbb{N}$, it always holds that 
    \begin{align}\label{eq:amortized_bound}
        &\frac{|\sum_{t=1}^T \eta_t\Delta_t|}{\sum_{t=1}^T \eta_t} \nonumber\\
        \leq &\left(1+\frac{\beta_2-\beta_1}{1-\beta_2} \frac{\sum_{t=1}^T \eta_t \beta_1^{t-1}}{\sum_{t=1}^T \eta_t} + \frac{(\beta_2-\beta_1)(1-\beta_1)}{(1-\beta_2)\sum_{t=1}^T \eta_t}
        \sum_{t=2}^T \left(\eta_t\frac{1-\beta_1^{t-1}}{1-\beta_1} -\sum_{i=1}^{T-t} \eta_{t+i} \beta_1^{i-1}\right)\ln{\frac{v_t}{v_1}}  \right)^{\frac{1}{2}}.
    \end{align}
    In particular, when $\beta_1=\beta_2$, it even holds that $|\Delta_t| \leq 1$. 
\end{lemma}
Note $\{v_t\}_{t=0}^\infty$ here only needs to satisfy a more general condition rather than to be the exact moving average of $g_t^2$. It can be applied to the practical scenario where a small positive constant $\epsilon$ is added to $\sqrt{\vv_t}$ in the denominator to improve the numerical stability of \adam. It is easy to verify that for $\vv_t$ in \Cref{alg:adamw}, we have that
\begin{align*}
    &(\sqrt{\vv_t}+\epsilon)^2 - \beta_2(\sqrt{\vv_{t-1}}+\epsilon)^2 \\
    \geq& (\vv_t -\beta_2 \vv_{t-1}) + 2\epsilon(\sqrt{\vv_t}-\beta_2\sqrt{\vv_{t-1}})\\
    =&(1-\beta_2)\vg_t^2 + 2\epsilon(\sqrt{\beta_2^2 \vv_{t-1} + (1-\beta_2)\vg_t^2}-\sqrt{\beta_2^2 \vv_{t-1}}) \\
    \geq& (1-\beta_2)\vg_t^2. 
\end{align*}
Therefore, for \adam with $\epsilon$, $v_t$ in \Cref{eq:amortized_bound} is always lower bounded, and if we further have an upper bound for gradients, then we can easily control the average update size of \adam. One nice property is that the upper bound only scales up logarithmically to $1/\epsilon$, instead of linearly, as the naive upper bound scales. 

\paragraph{Relationship of $\ell_\infty$ norm and hyperparameters}
Another application of \Cref{lem:amortized_update_bound} is to provide a tight upper bound for the norm of iterates for any setting, \emph{e.g.}, before convergence or even when the gradient is stochastic. In particular, when the learning rate does not change over steps, we have the following upper bound whose proof is in \Cref{sec:iterate_norm_bound}. 
% \zhiyuan{add comments and a link to empirical justification for bounded ratio of $\vv$}
\begin{lemma}\label{lem:iterate_norm_bound}
For any coordinate $j\in[d]$, for \adamw with constant learning rate $\eta$ and weight decay factor $\lambda$, with $C\triangleq\max\limits_{1\le t\le T}\abs{\ln{\frac{\vv_{t,j}}{\vv_{1,j}}}}$, it holds that
    \begin{equation}
        \lambda\abs{\vx_{T,j}}-1 \leq  (1-\lambda\eta)^T \lambda\abs{\vx_{0,j}} +\frac{\lambda \eta(\beta_2-\beta_1)\left[2C + \beta_1^{T} + (1-\lambda\eta)^{T} \right]}{2(1-\beta_2)|1-\lambda\eta-\beta_1|}.%1+ \lambda(1-\lambda\eta)^T \norm{\vx_{0}}_\infty + C\frac{\lambda \eta(\beta_2-\beta_1)}{(1-\beta_2)|1-\lambda\eta-\beta_1|}
    \end{equation}
\end{lemma}
When $\beta_1=\beta_2$, we only need $T=\Omega\left(\frac{\log\norm{\vx_0}_\infty}{\lambda\eta}\right)$ to guarantee that $\abs{\vx_T,j}$ is no larger than $\frac{1}{\lambda}$ for any $\lambda\eta\le 1$. However, when $\beta_1 < \beta_2$ and $\beta_1 < 1-\lambda\eta$, the dominating term on the right-hand side is $C\cdot \frac{\eta\lambda(\beta_2-\beta_1)}{(1-\beta_2)(1-\eta\lambda-\beta_1)}$. Assuming $C=O(1)$,  it also requires $\lambda\eta \ll 1-\beta_2<1-\beta_1$ or $\lambda\eta<1-\beta_2 \approx 1-\beta_1$ to ensure the remaining term is small.
\section{Experiments}\label{sec:experiments}
In this section, we run experiments to verify the theoretical claims. In \Cref{sec:ptb_exp}, we show that the $\ell_\infty$ norm of iterates by \adamw can converge below $\frac{1}{\lambda}$ as shown in \Cref{thm:main} even when the function is non-convex. In \Cref{sec:synthetic_exp}, we show that steepest descent w.r.t. $\ell_\infty$ norm works better than w.r.t. $\ell_2$ norm for a specific function, which has better properties under $\ell_\infty$ geometry.  
\subsection{Language modeling task on PTB}\label{sec:ptb_exp}
We train a small two-layer transformer for language modeling task on the Penn Treebank dataset (PTB)~\citep{marcus1993building} based on the implementation provided by~\citet{kunstner2023noise}. We train the model in full batch without dropout in order to get deterministic gradients and follow the constant learning rate setting for the total $12800$ epochs. The learning rate $\eta$ is $\sqrt{10}\times10^{-3}$\footnote{We follow the tuning process in~\citet{kunstner2023noise} for \adamw and $\sqrt{10}\times10^{-3}$ achieves the best training performance. $\sqrt{10}\times10^{-3}$ also achieves the best training performance for full-batch \adam in~\citet{kunstner2023noise}. }. For each setting of $\beta_1$, $\beta_2$, we use \adam and \adamw with weight decay coefficient $\lambda=1,2$ to compare the $\ell_\infty$ norm for iterates in each optimizer. We employ the standard implementation in PyTorch but set $\epsilon$ to be $10^{-16}$ in \adam and \adamw rather than $0$ to avoid division by zero error because the gradient of some coordinates is $0$ in the first iteration. Each run is repeated for $2$ random seeds to show the robustness of our claim. More details can be found in \Cref{sec:exp_details}.

\begin{figure}
    \centering
    \begin{subfigure}[b]{0.24\textwidth}
        \includegraphics[width=\textwidth]{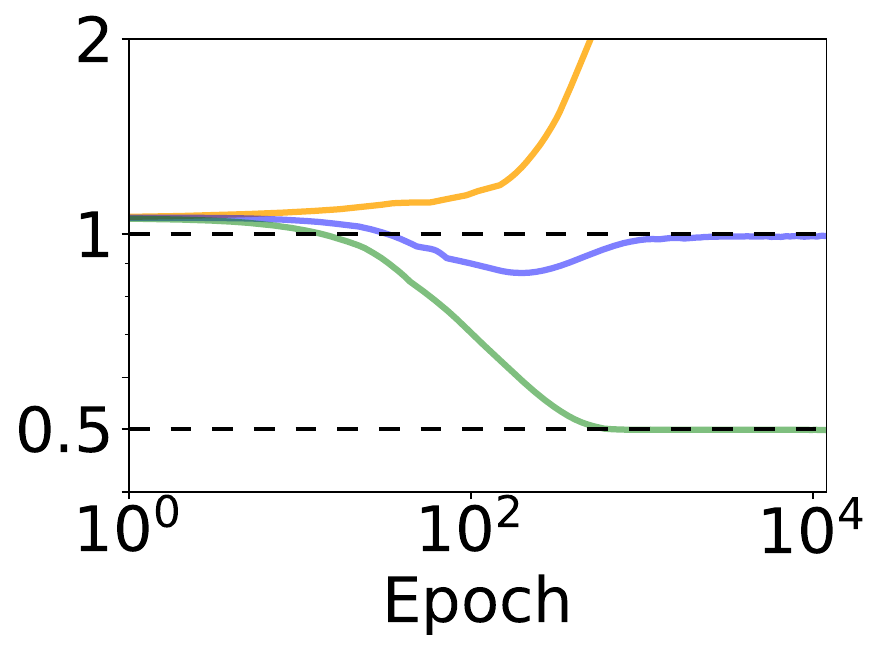}
        \caption{$\beta_1=\beta_2=0.99$}
        \label{fig:success_2}
    \end{subfigure}
    \begin{subfigure}[b]{0.24\textwidth}
        \includegraphics[width=\textwidth]{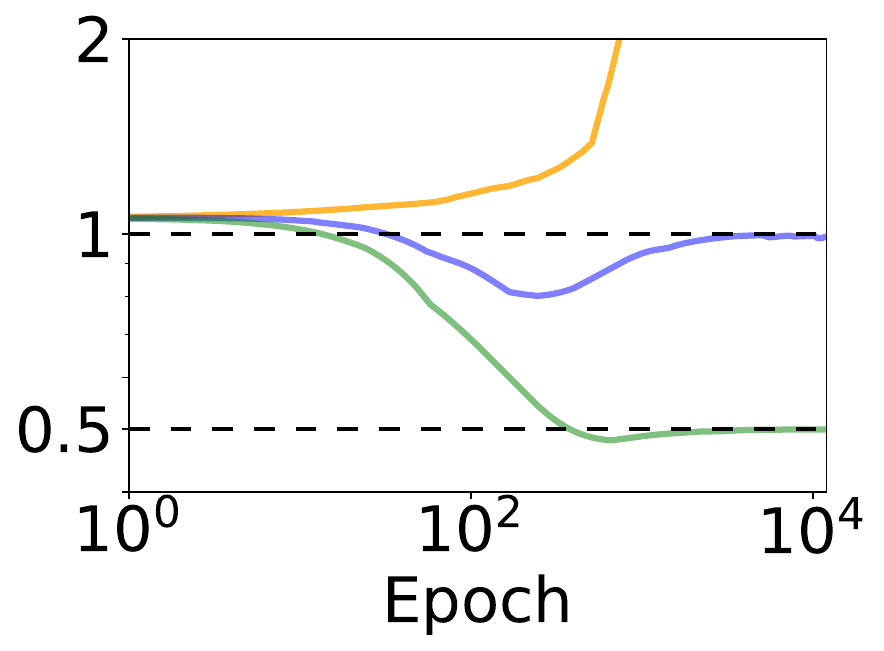}
        \caption{$\beta_1=\beta_2=0.999$}
        \label{fig:success_3}
    \end{subfigure}
    \begin{subfigure}[b]{0.24\textwidth}
        \includegraphics[width=\textwidth]{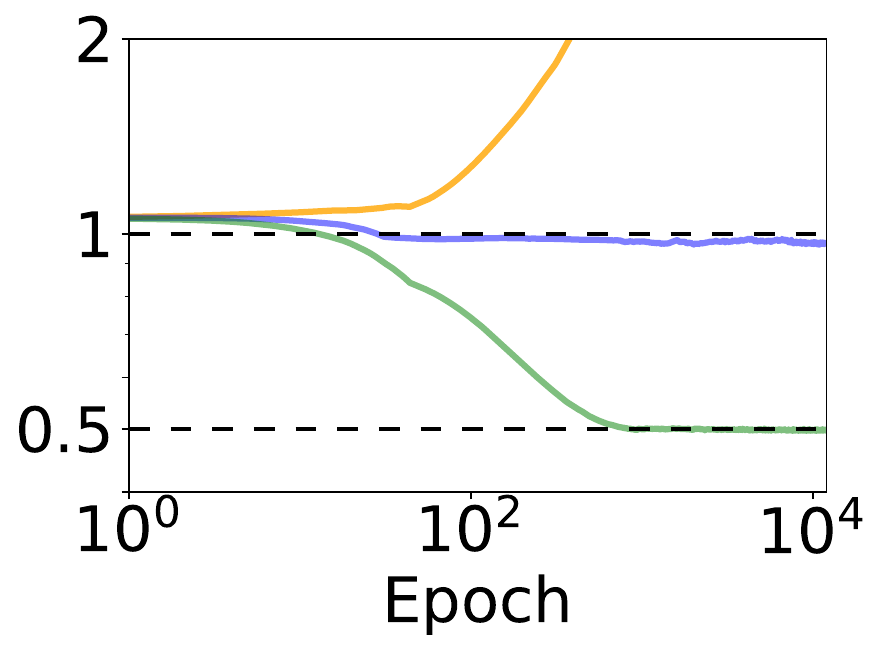}
        \caption{$\beta_1=0.9$, $\beta_2=0.95$}
        \label{fig:success_1}
    \end{subfigure}
    \begin{subfigure}[b]{0.24\textwidth}
        \includegraphics[width=\textwidth]{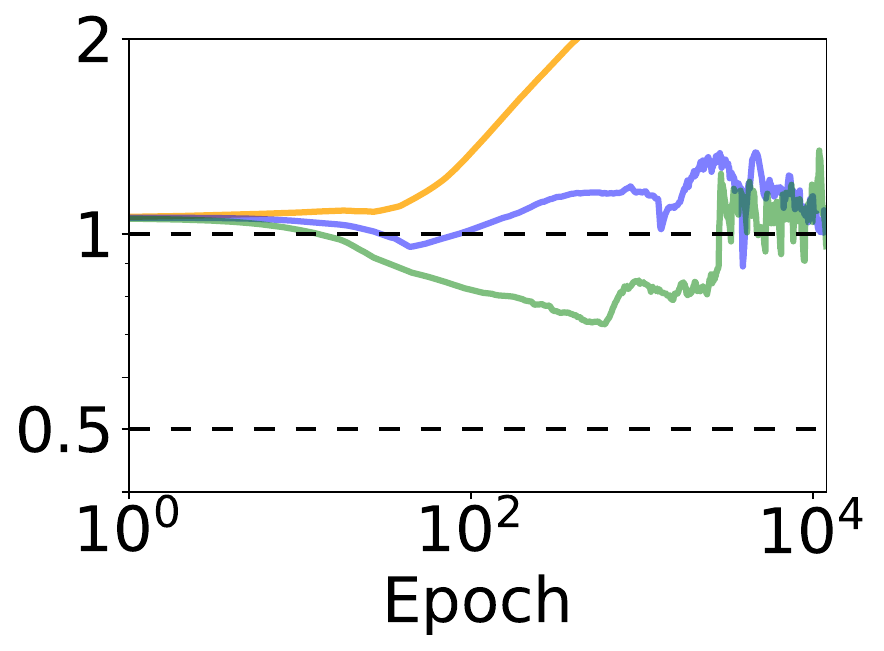}
        \caption{$\beta_1=0.9, \beta_2=0.999$}
        \label{fig:failure}
    \end{subfigure}

    \begin{subfigure}[b]{0.6\textwidth}
        \includegraphics[width=\textwidth]{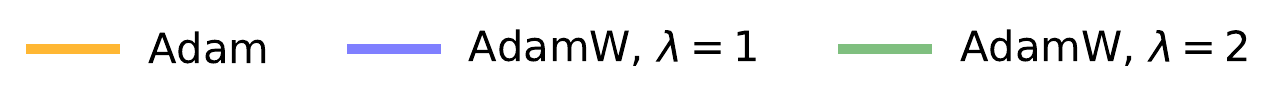}
        \phantomcaption
    \end{subfigure}
    \caption{The $\ell_\infty$ norm of parameters during the training process of language modeling task on PTB. 
    The complete results for \adam are in \Cref{fig:ptb_seed_0_full}. As predicted by \Cref{lem:iterate_norm_bound}, $\ell_\infty$ norm can be bounded by $\frac{1}{\lambda}$ when $\beta_1=\beta_2$ or $\lambda\eta \ll 1-\beta_2<1-\beta_1$. However, for the default setting $\beta_1=0.9$ and $\beta_2=0.999$, the $\ell_\infty$ norm of \adamw may not be bounded by $\frac{1}{\lambda}$ because $1-\beta_2<\lambda\eta<1-\beta_1$.}
    \label{fig:ptb_seed_0}
\end{figure}

From the discussion in \Cref{lem:iterate_norm_bound}, it requires either $\beta_1\approx\beta_2$ or $\lambda\eta \ll 1-\beta_2$ when $\beta_1 < \beta_2$ in order for $\norm{\vx_T}_\infty$ to be bounded by $\frac{1}{\lambda}$. To verify the first case, we employ the two hyperparamter settings where $\beta_1=\beta_2=0.99$ and $\beta_1=\beta_2=0.999$. To verify the second case, we employ the hyperparameter setting where $\beta_1=0.9$ and $\beta_2=0.95$. 
The $\ell_\infty$ norm of iterates for one random seed are shown in \Cref{fig:success_1,fig:success_2,fig:success_3}. In order to show the details around $0.5$ and $1$, we truncate the range of y-axis and the full result is plotted in \Cref{fig:ptb_seed_0_full} in \Cref{sec:exp_details}. 
In all these three settings, the $\ell_\infty$ norm of iterates in \adam keeps increasing while the $\ell_\infty$ norm of iterates in \adamw is constrained below $1$ and $\frac{1}{2}$ for $\lambda=1$ and $2$ respectively. The results for another random seed show similar pattern and are plotted in \Cref{fig:ptb_seed_1} and \Cref{fig:ptb_seed_1_full}. 

We also show that the condition is necessary for empirical study by training with default $\beta_1=0.9$ and $\beta_2=0.999$. Now $1-\beta_2 < \lambda \eta$ and $\beta_1 \neq \beta_2$, which breaks the condition. The $\ell_\infty$ norm of iterates are shown in \Cref{fig:failure}. The $\ell_\infty$ norm of \adamw can not be constrained by $1$ and $\frac{1}{2}$ for $\lambda=1$ and $2$. 

To study the phenomenon when the full batch assumption is removed, we repeat the same experiment with the batch size being half of the training dataset size. The results are shown in \Cref{fig:ptb_half_seed_0}. Note our norm upper bound~\Cref{lem:iterate_norm_bound} still holds in this case because it allows any sequence of gradients.  Interestingly, similar to the full batch case~\Cref{fig:ptb_seed_0},  $\ell_\infty$ norm still approaches $\frac{1}{\lambda}$, suggesting that our theory still provides useful prediction even for stochastic \adamw, if the noise is small.
\begin{figure}
    \centering
    \begin{subfigure}[b]{0.24\textwidth}
        \includegraphics[width=\textwidth]{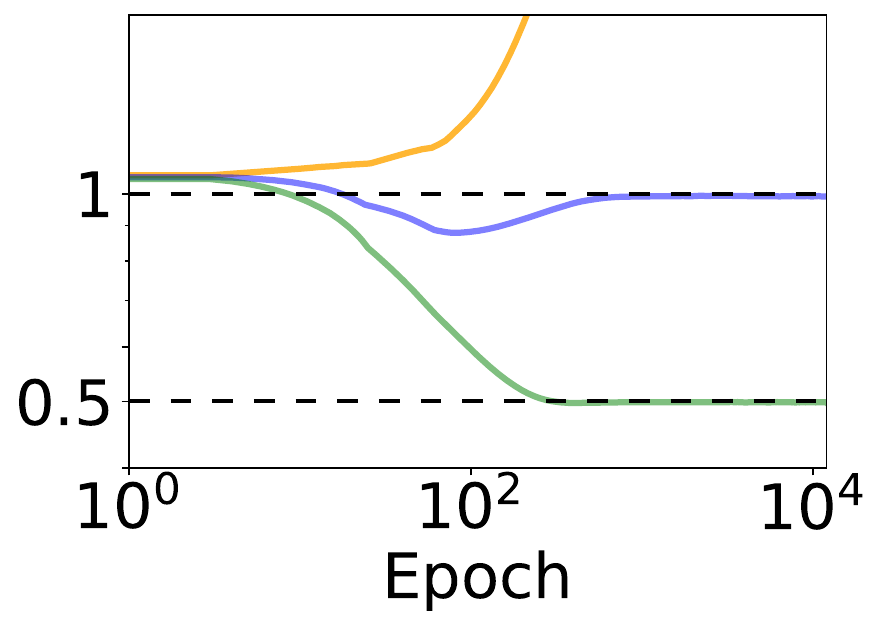}
        \caption{$\beta_1=\beta_2=0.99$}
        % \label{fig:success_2}
    \end{subfigure}
    \begin{subfigure}[b]{0.24\textwidth}
        \includegraphics[width=\textwidth]{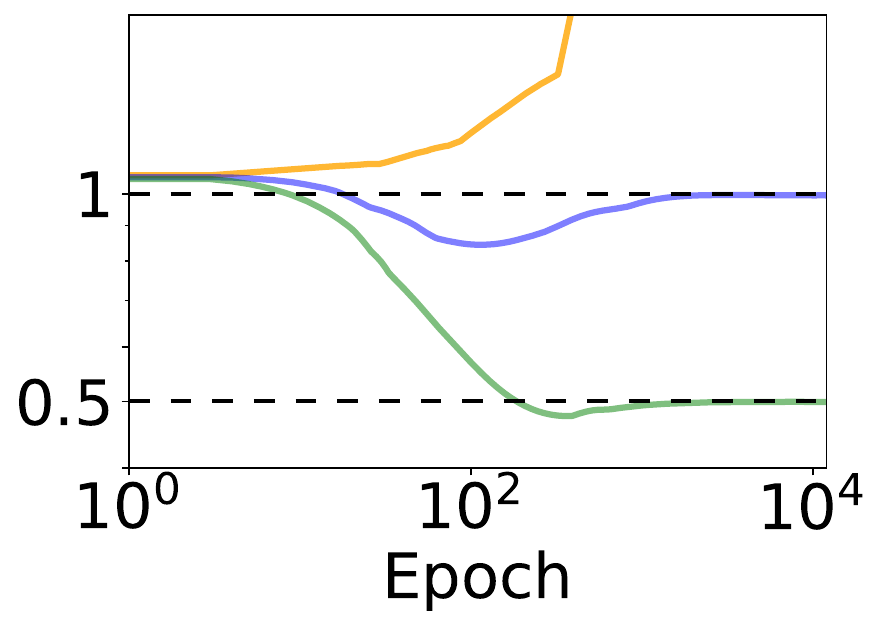}
        \caption{$\beta_1=\beta_2=0.999$}
        % \label{fig:success_3}
    \end{subfigure}
    \begin{subfigure}[b]{0.24\textwidth}
        \includegraphics[width=\textwidth]{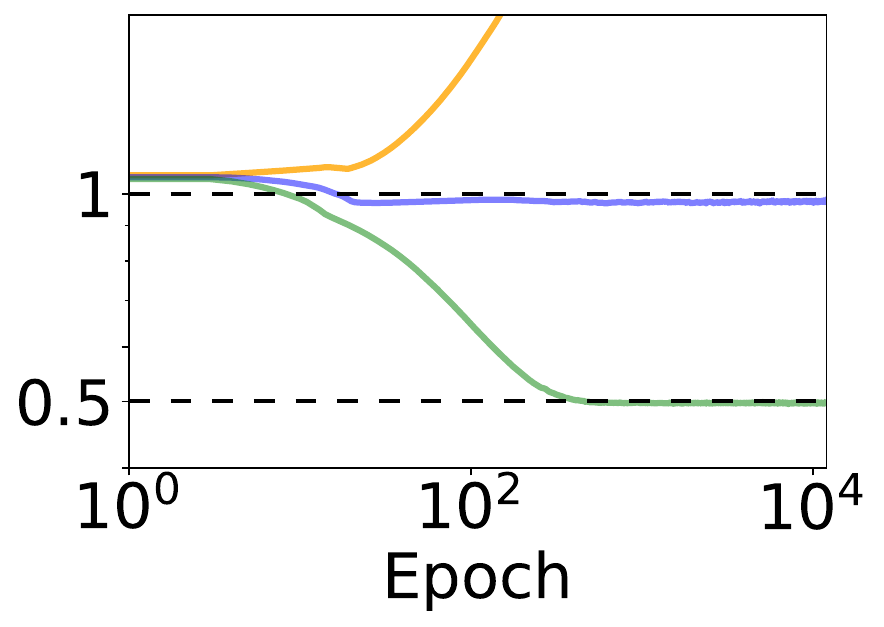}
        \caption{$\beta_1=0.9$, $\beta_2=0.95$}
        % \label{fig:success_1}
    \end{subfigure}
    \begin{subfigure}[b]{0.24\textwidth}
        \includegraphics[width=\textwidth]{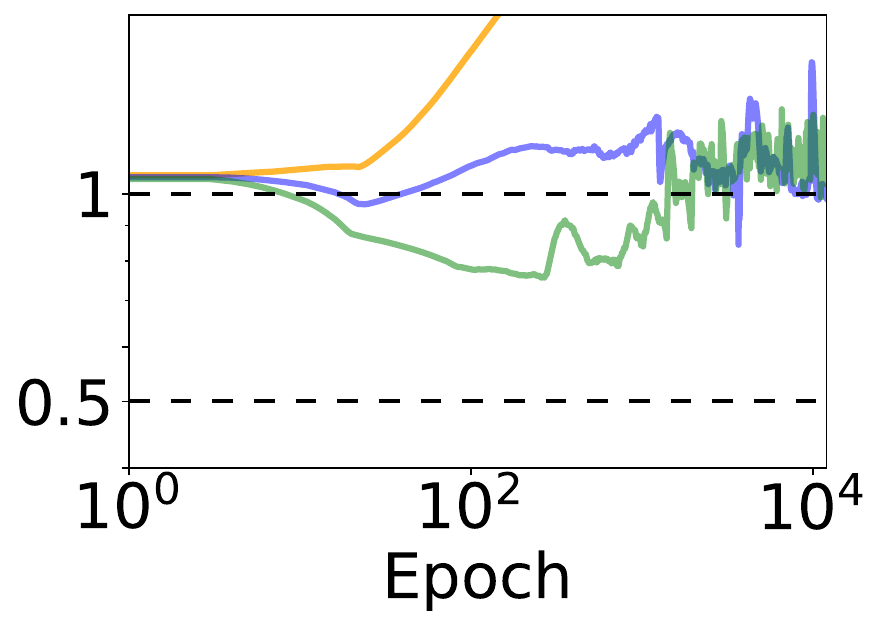}
        \caption{$\beta_1=0.9, \beta_2=0.999$}
        % \label{fig:failure}
    \end{subfigure}

    \begin{subfigure}[b]{0.6\textwidth}
        \includegraphics[width=\textwidth]{figures/legend.pdf}
        \phantomcaption
    \end{subfigure}
    \caption{$\ell_\infty$ norm of parameters when the batch size is half of the entire training set. Our norm upper bound~\Cref{lem:iterate_norm_bound} still holds in this case.}
    \label{fig:ptb_half_seed_0}
\end{figure}
\subsection{A synthetic problem}\label{sec:synthetic_exp}

As mentioned in \Cref{subsec:convex_convergence}, we construct the loss function \Cref{eq:synthetic_example} for some $\vx^*\in\mathbb{R}^{100}$. The first $10$ coordinates of $\vx^*$ is $1$ while the rest $90$ coordinates are uniformly sampled between $[-1,1]$. With such initialization, the $\ell_\infty$ norm of $\vx^*$ is always upper bounded by $1$ but the $\ell_2$ norm can be as large as $10$. 

We want to verify that an optimization algorithm can take advantage when it employs the norm that is more suitable for the loss function. So we implement normalized steepest descent and steepest descent with $\ell_\infty$ norm and $\ell_2$ norm. In order to test the effect of weight decay, we also implement normalized steepest descent with weight decay for both norms. In order to be able to reach the global minimizer, the weight decay factor $\lambda$ is set to be the lower bound of $\frac{1}{\norm{\vx^*}}$, which is $1$ for $\ell_\infty$ norm and $0.1$ for $\ell_2$ norm. The learning rate can be set according to theoretical results. The learning rate for steepest descent is constant $\frac{1}{H}$ in which $H$ is the smoothness coefficient. By \Cref{thm:1/t_convergence}, the learning rate for normalized steepest descent with weight decay factor $\lambda$ should be set as $\eta_t=\frac{2}{\lambda(t+1)}$. We also use the same learning rate schedule for normalized steepest descent without weight decay. 

All the algorithms receive the same initialization $\vx_0$, whose coordinate is uniformly initialized from $[-5,5]$. The results for the first $100$ iterations are shown in \Cref{fig:synthetic_exp}. The steepest descent w.r.t. $\ell_\infty$ norm always performs better than the steepest descent w.r.t. $\ell_2$ norm no matter whether the update is normalized or not. For both norms, the performance of the normalized steepest descent is improved when weight decay is activated.

\shuo{Maybe comment more on the advantage of $\ell_\infty$ norm. }
\begin{figure}
    \centering
    \includegraphics[width=0.5\linewidth]{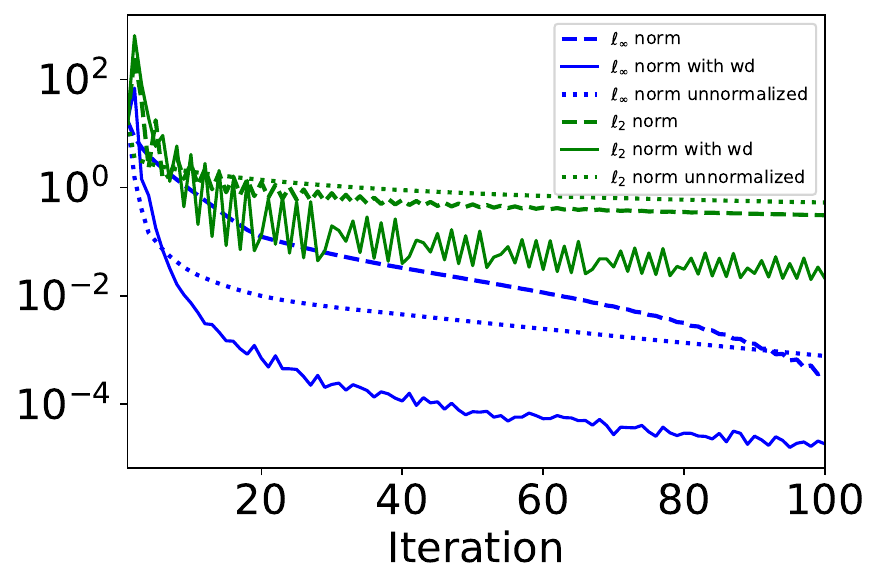}
    \caption{For both $\ell_2$ and $\ell_\infty$ norm, we plot training loss of normalized steepest descent w. and w.o. weight decay and unnormalized steepest descent over the quadratic loss $g(\vx) = \sum_{i=1}^{100} \frac{(\vx_i-\vx^*_i)^2}{i^2}$. When weight decay is turned on, it is set as $
    \frac{1}{\norm{\vx^*}}$ to preserve the optimal value even with the norm constraints~\Cref{thm:any_rate_convergence}. We find that \textbf{ $\ell_\infty$ norm always outperforms $\ell_2$ norm} regardless of the usage of weight decay and irrespective of whether the steepest descent method is normalized. The usage of \textbf{weight decay accelerates the optimization} for both $\ell_\infty$ norm and $\ell_2$ norm.}
    \label{fig:synthetic_exp}
\end{figure}

\section{Related Work}\label{sec:related_works}

\paragraph{Adaptive Methods:} While Stochastic Gradient Descent~\citep{robbins1951stochastic} remains popular for optimizing deep learning models like ResNet~\citep{he2016deep}, only adaptive methods can efficiently train recently-emerged large language models~\citep{zhang2020adaptive}. There has been a fruitful amount of research on adaptive gradient method, including AdaGrad~\citep{duchi2011adaptive}, RMSProp~\citep{tieleman2012lecture},  AdaDelta~\citep{zeiler2012adadelta}, Adam~\citep{kingma2014adam}, AdaFactor~\citep{shazeer2018adafactor}, AMSGrad~\citep{reddi2018convergence}, AdaBound~\citep{luo2019adaptive}, Lion~\citep{chen2023symbolic}, \emph{etc}. Recently there have been also adaptive methods attempting to accelerate by leveraging the second-order information, \emph{e.g.}, AdaHessian~\citep{yao2021adahessian} and Sophia~\citep{liu2023sophia}. However, most algorithms that are able to train large language models adopt coordinate-wise adaptivity. In contrast, stochastic gradient descent, even equipped with global gradient norm clipping, cannot match the performance of coordinate-wise adaptive algorithms on language tasks~\citep{li2022robust}. Previous work has given convergence rate for RMSProp and Adam under different assumptions~\citep{chen2018convergence,zou2019sufficient,shi2021rmsprop,guo2021novel,defossez2022simple,zhang2022adam}.

Our work shows that \adamw and \signgd with weight decay converge to the same point assuming convergence. \citet{balles2018dissecting, kunstner2023noise} point out that the similarity with \signgd largely accounts for the advantage of \adam over \sgd.  
Moreover, when \signgd is equipped with momentum which is one key component of \adam, it can achieve comparable empirical results with \adam for various tasks~\citep{balles2018dissecting, kunstner2023noise,bernstein2018signsgd,crawshaw2022robustness}.

\paragraph{Role of Weight Decay:} The usage of weight decay, which refers to shrinking the parameter by a small constant fraction, can be dated back to the 1980s~\citep{rumelhart1986learning,hinton1987learning}. It has been recognized as a standard trick to improve the generalization performance of neural
networks~\citep{krogh1991simple,bos1996using} for a long time. \citet{krizhevsky2012imagenet} first noticed that weight decay can sometimes accelerate optimization in deep learning. For modern architectures equipped with normalization layers, \emph{e.g.}, BatchNorm~\citep{ioffe2015batch} and LayerNorm~\citep{ba2016layer}, only the direction of the parameters before normalization layers matters, rather than their norms. Turning on weight decay in such settings changes the effective learning rate of the parameters~\citep{hoffer2018norm,arora2018theoretical,zhang2018three,li2019exponential,li2020reconciling}.

Though weight decay is equivalent to $\ell_2$ regularization for SGD, for steepest descent methods with general norms and adaptive methods like \adam, they lead to different optimization trajectories~\citep{loshchilov2017decoupled,zhang2018three,zhuang2022understanding}. The empirical benefit of weight decay over $\ell_2$ regularization when they are different is not well-understood in theory. 

\paragraph{Implicit Regularization:} Our main result~\Cref{thm:main} shows that \adamw regularizes the $\ell_\infty$ norm of the learned solution implicitly through modifying the optimization dynamics, rather than directly modifying the objective, like \adam with $\ell_2$ regularization. This kind of behavior is termed \emph{implicit regularization} or \emph{implicit bias} of optimization algorithms. Though there has been a large volume of works studying the implicit bias of (Stochastic) GD and its non-adaptive variants, including settings related to max margin~\citep{soudry2018implicit,gunasekar2018characterizing,nacson2019lexicographic,nacson2019convergence,ji2018gradient,lyu2019gradient}, initialization with norm~\citep{gunasekar2017implicit,arora2019implicit,li2019towards,lyu2021gradient}, kernel regime~\citep{jacot2018neural,arora2019fine,arora2019exact}, and flatness~\citep{blanc2019implicit,damian2021label,li2021happens,arora2022understanding,li2022fast,wen2022sharpness,damian2022self}, very few results are known about the implicit bias of adaptive methods like \adam. 
\citet{wilson2017marginal} shows that for linear regression problem, adaptive methods can converge to a solution whose elements have the same magnitude under certain conditions. Their converged solution thus has small $\ell_\infty$ norm while the converged solution of non-adaptive methods is known to have the smallest $\ell_2$ norm among all the global minimizers. \citet{wang2021implicit} shows that \adam behaves similarly to non-adaptive methods like GD when the cross-entropy loss converges to $0$ due to the positive numerical stability hyperparameter $\epsilon$ in the denominator of \adam's update rule. The theoretical derivation by \citet{cattaneo2023implicit} argues that \adam tends to find interpolating solutions with small $\ell_1$ norm. 

The concurrent work by \citet{chen2023lion} is arguably the most related work to us, where the recently discovered optimization algorithm by auto-search, Lion~\citep{chen2023symbolic}, is elegantly generalized to a family of algorithms, Lion-$\mathcal{K}$, where $\mathcal{K}$ is some convex function. When $\mathcal{K}$ is chosen to be the dual norm and momentum in Lion-$\mathcal{K}$ is turned off, Lion-$\mathcal{K}$ becomes the normalized steepest descent. Their analysis shows that even with momentum, the steepest normalized descent with weight decay can be viewed as optimization under the original norm constraint. However, in any Lion-$\mathcal{K}$ algorithm, the update at one step $t$ only depends on past iterates through first-order momentum $\vm_t$. Their analysis cannot be applied to \adamw because \adamw cannot be written in the form of Lion-$\mathcal{K}$ for any convex function $\mathcal{K}$. To see this, simply note that the update of Lion-$\mathcal{K}$ for a fixed $\mathcal{K}$ is completely determined by  $\vg_t, \vm_t$ and $\vx_t$ while the update of \adamw can still be different if the second order momentum $\vv_t$ is different. In terms of proof technique, \citet{chen2023lion} constructs the Lyapunov function while we directly characterize the KKT point and connect the converged point to KKT point through the weighted average update. 
\section{Discussion and Future Works}\label{sec:discussion}

This work focuses on the implicit bias of \adamw in the deterministic (or full-batch) case. Though our upper bound on the average update size of \adam holds unconditionally on the input gradients, regardless of stochasticity or not, it is unlikely that the $\frac{1}{\lambda}$ upper bound can be reached when there is large gradient noise, especially when $\beta_2$ is very close to $1$. In that case, the denominator of the update of \adamw is roughly the square root of the square of the expected gradient plus some additional gradient variance term, which strictly dominates the expected gradient in the numerator. \citet{malladi2022sdes} uses Stochastic Differential Equation (SDE) approximation to model the trajectories of \adam in such regime and empirically tests the implication of SDE approximation, namely the square root scaling rule.

The most interesting future direction is to understand in what sense the optimization advantage of coordinate-wise adaptive methods like \adam over standard SGD for language modeling tasks can be explained by the conjecture implied by the findings of \citet{kunstner2023noise}, that \emph{the loss function for language modeling tasks has better properties under $\ell_\infty$ geometry}. It would be interesting to understand if the loss landscape in real-world applications shares common properties with our toy quadratic example and induces similar results in~\Cref{fig:synthetic_exp}.

Another important future direction is to provide non-asymptotic convergence rates for \adamw in both convex and non-convex settings. 
\section{Conclusions}\label{sec:conlusions}

We make the first step towards understanding the benefit of \adamw over \adam with $\ell_2$ regularization by characterizing the implicit bias of \adamw, \emph{i.e.}, it can only converge to KKT points of the $\ell_\infty$ norm constrained optimization problem. There are two main insights behind this result: (1) \adam is a smoothed version of \signgd, which is the normalized steepest descent w.r.t. $\ell_\infty$ norm; (2) for any norm, the corresponding normalized steepest descent with weight decay is essentially Frank-Wolfe over the corresponding norm ball, which is known to perform constrained optimization. Our main technical contribution is a tight upper bound of the average update size of \adam updates. We test its prediction on the relationship between the $\ell_\infty$ norm of the parameters and the \adamw hyperparameters $\eta,\lambda$, $\beta_1$, $\beta_2$ on a language modeling task.

\bibliography{all}
\bibliographystyle{iclr2023_conference}

\appendix
\newpage {
\hypersetup{linkcolor=black}
\tableofcontents
}
\section{Omitted Proofs in \Cref{sec:warm_up}}\label{sec:normalized_sd}
In this section, we provide the omitted proofs in \Cref{sec:warm_up}, which shows the iterates and the converged solution by normalized steepest descent with decoupled weight decay before diving into the analysis on \adamw. In \Cref{sec:convergence_to_ball}, we prove that the iterates will enter or stay in the norm ball with radius $\frac{1}{\lambda}$ for any normalized update. In \Cref{sec:convergence_nsd}, we prove that the iterates of normalized steepest descent with weight decay will converge to the constrained minimizer of $L(\vx)$ in the same ball with proper learning rates.
\subsection{Omitted proofs for convergence into norm ball with bounded update}\label{sec:convergence_to_ball}

\begin{proof}\Cref{lem:convergence_to_ball}
    We prove by induction that $\norm{\vx_t} \leq \frac{1}{\lambda}+\prod_{i=1}^{t}(1-\lambda \eta_i) \left(\norm{\vx_0} - \frac{1}{\lambda} \right)$. 
    \begin{align*}
        \norm{\vx_t} - \frac{1}{\lambda} &= \norm{(1-\lambda \eta_{t}) \vx_{t-1} - \eta_{t} \vDelta_{t}}  - \frac{1}{\lambda}\\
        &\leq (1-\lambda \eta_{t}) \norm{\vx_{t-1}} + \eta_{t} \norm{\vDelta_{t}} - \frac{1}{\lambda} \\
        &\leq (1-\lambda \eta_{t}) \norm{\vx_{t-1}} + \eta_{t} - \frac{1}{\lambda} \\
        & = (1-\lambda \eta_{t}) \left(\norm{\vx_{t-1}} - \frac{1}{\lambda} \right)\\
        & \leq \prod_{i=1}^{t} (1-\lambda \eta_i) \left(\norm{\vx_0} - \frac{1}{\lambda}\right).
    \end{align*}
    When $\norm{\vx_0} > \frac{1}{\lambda}$, we have that
    \begin{align*}
        \norm{\vx_t} - \frac{1}{\lambda} &\leq \prod_{i=1}^{t} (1-\lambda \eta_i) \left(\norm{\vx_0} - \frac{1}{\lambda}\right) \\
        &\leq \prod_{i=1}^{t} \exp(-\lambda \eta_i) \left(\norm{\vx_0} - \frac{1}{\lambda}\right) \\
        &=\exp(-\lambda \sum_{i=1}^{t} \eta_i) \left(\norm{\vx_0} - \frac{1}{\lambda}\right).
    \end{align*}
    When $\norm{\vx_0} \leq \frac{1}{\lambda}$, $\norm{\vx_t} -\frac{1}{\lambda} \leq 0$. This completes the proof.
\end{proof}
\subsection{Omitted proofs for convergence to constrained minimizer with proper learning rates}\label{sec:convergence_nsd}

\begin{proof}[Proof of \Cref{lem:one_step_nsd}]

For normalized steepest descent update $\vDelta_t$ from \Cref{eq:def_nsd_wd}, 
\begin{align*}
        \nabla L(\vx_{t-1})^\top (\vx_{t}-\vx_{t-1}) &= -\eta_t \nabla L(\vx_{t-1})^\top \vDelta_t - \lambda \eta_t\nabla L(\vx_{t-1})^\top \vx_{t-1} \\
        &=-\eta_t \norm{\nabla L(\vx_{t-1})}_* - \lambda \eta_t \nabla L(\vx_{t-1})^\top (\vx_{t-1} - \vx^*) - \lambda \eta_t \nabla L(\vx_{t-1})^\top \vx^*\\
        &\leq -\eta_t \norm{\nabla L(\vx_{t-1})}_* - \lambda \eta_t \left(L(\vx_{t-1}) -L(\vx^*) \right) + \lambda \eta_t \norm{\nabla L(\vx_{t-1})}_* \norm{\vx^*}\\
        &\leq - \lambda \eta_t \left(L(\vx_{t-1}) -L(\vx^*) \right),
    \end{align*}
where the first inequality we use convexity of $L$ and the second inequality uses $\norm{\vx^*}\le 1$.

Since the gradient of $L$ is $H$-lipschitz, by Taylor expansion, we have that 
    \begin{align*}
        &L(\vx_{t}) - L(\vx_{t-1})\\ =& \int_0^1 \nabla L\left(\vx_{t-1} + \alpha (\vx_{t} - \vx_{t-1}) \right)^\top(\vx_{t} -\vx_{t-1}) \,d\alpha \\
        =& \nabla L(\vx_{t-1})^\top (\vx_{t}-\vx_{t-1}) + \int_0^1 \left(\nabla L\left(\vx_{t-1} + \alpha (\vx_{t} - \vx_{t-1}) \right) -\nabla L(\vx_{t-1}) \right)^\top(\vx_{t} -\vx_{t-1}) \,d\alpha \\
        \leq& - \lambda \eta_t \left(L(\vx_{t-1}) -L(\vx^*) \right) + \int_0^1 \norm{\nabla L\left(\vx_{t-1} + \alpha (\vx_{t} - \vx_{t-1}) \right) -\nabla L(\vx_{t-1})}_* \norm{\vx_{t}-\vx_{t-1}} \,d\alpha \\
        \leq& - \lambda \eta_t \left(L(\vx_{t-1}) -L(\vx^*) \right) + \int_0^1 H \alpha \norm{\vx_{t}-\vx_{t-1}}^2 \,d\alpha \\
        =& - \lambda \eta_t \left(L(\vx_{t-1}) -L(\vx^*) \right) +\frac{H}{2} \norm{\vx_{t}-\vx_{t-1}}^2.
    \end{align*}
    Because the update $\vDelta_t$ is normalized and thus have unit norm by definition, it holds that 
    \begin{align*}
        \norm{\vx_{t}-\vx_{t-1}}^2 = \norm{-\eta_t \vDelta_t - \lambda \eta_t \vx_{t-1}}^2 
        \leq \eta_t^2\left(\norm{\vDelta_t} + \lambda \norm{\vx_{t-1}} \right)^2 
        \leq \eta_t^2\left(1 + \lambda \norm{\vx_{t-1}} \right)^2. 
    \end{align*}
    Finally, we conclude that 
    \[L(\vx_{t}) - L(\vx^*) \leq (1-\lambda \eta_t)(L(\vx_{t-1}) - L(\vx^*)) + \frac{H}{2} \eta_t^2\left(1 + \lambda \norm{\vx_{t-1}} \right)^2.\]
\end{proof}

\begin{proof}[Proof of \Cref{thm:any_rate_convergence}]
    The proof of \Cref{thm:any_rate_convergence} is a direct application of  \Cref{lem:lr_convergence} on the one-step descent lemma~\Cref{lem:one_step_nsd}.
\end{proof}

\begin{lemma} \label{lem:lr_convergence}
    Assume that $\eta_t \geq 0$, $\lim_{t \rightarrow \infty} \eta_t =0$ and $\sum_{t=1}^\infty \eta_t = \infty$. $C$ is any positive number and $a_0 \geq 0$. If the sequence $\{a_t \}_{t=0}^\infty$ satisfies that $a_{t} \leq (1-\eta_t) a_{t-1} + C \eta_t^2$, then $\lim_{t \rightarrow \infty} a_t = 0$.
\end{lemma}
\begin{proof}[Proof of \Cref{lem:lr_convergence}]
First we show by induction that $a_t \leq a_0 \exp\left(-\sum_{i=1}^{t} \eta_i\right) + C \sum_{i=1}^{t} \eta_i^2 \exp\left(-\sum_{j=i+1}^{t} \eta_j \right)$.
    \begin{align*}
        a_{t} & \leq (1-\eta_t) a_{t-1} + C \eta_t^2 \\
        & \leq \exp\left(-\eta_t\right) a_{t-1} + C\eta_t^2 \\
        & \leq \exp\left(-\eta_t\right) \left[a_0 \exp\left(-\sum_{i=1}^{t-1} \eta_i\right) + C \sum_{i=1}^{t-1} \eta_i^2 \exp\left(-\sum_{j=i+1}^{t-1} \eta_j\right) \right] + C\eta_t^2 \\
        &=a_0 \exp\left(-\sum_{i=1}^{t} \eta_i\right)+C \sum_{i=1}^{t-1} \eta_i^2 \exp\left(-\sum_{j=i+1}^{t} \eta_j\right) + C\eta_t^2 \\
        &= a_0 \exp\left(-\sum_{i=1}^{t} \eta_i\right)+C \sum_{i=1}^{t} \eta_i^2 \exp\left(-\sum_{j=i+1}^{t} \eta_j\right).
    \end{align*}
Because $\sum_{t=1}^\infty \eta_t = \infty$, $\lim_{t \rightarrow \infty} a_0 \exp\left(-\sum_{i=1}^{t} \eta_i\right) =0$. In order to show $\lim_{t \rightarrow \infty} a_t = 0$, it's sufficient to show $\lim_{t \rightarrow \infty} \sum_{i=1}^{t} \eta_i^2 \exp\left(-\sum_{j=i+1}^{t} \eta_j \right) =0$. 

For any $\epsilon > 0$, $\eta^*$ is chosen such that $\eta^* e^{\eta^*} = \frac{\epsilon}{2}$. There exists $\tau \in \mathbb{N}^+$ such that $\eta_i \leq \eta^*$ for $i \geq \tau$. We choose $T$ such that $\exp(-\sum_{j=\tau}^{T} \eta_j) \leq \frac{\epsilon}{2\sum_{i=1}^{\tau-1} \eta_i^2}$. Then for any $t \geq T$, we have the following inequalities
\begin{align*}
    \sum_{i=1}^{\tau-1} \eta_i^2 \exp\left(-\sum_{j=i+1}^{t} \eta_j \right) & \leq \sum_{i=0}^{\tau-1} \eta_i^2 \exp\left(-\sum_{j=\tau}^{t} \eta_j \right) \\
    & \leq \left(\sum_{i=1}^{\tau-1} \eta_i^2 \right) \exp\left(-\sum_{j=\tau}^{T} \eta_j \right)\\
    & \leq \frac{\epsilon}{2},
\end{align*}
\begin{align*}
    \sum_{i=\tau}^{t} \eta_i^2 \exp\left(-\sum_{j=i+1}^{t} \eta_j \right) &\leq \eta^* \sum_{i=\tau}^{t} \eta_i \exp\left(-\sum_{j=i+1}^{t} \eta_j \right) \\
    &\leq \eta^* \sum_{i=\tau}^{t} \left(\exp(\eta_i) - 1\right) \exp\left(-\sum_{j=i+1}^{t} \eta_j \right) \\
    &=\eta^* \sum_{i=\tau}^{t}\exp(\eta_i) \left( 1 - \exp(-\eta_i)\right) \exp\left(-\sum_{j=i+1}^{t} \eta_j \right) \\
    & \leq \eta^* \sum_{i=\tau}^{t-1}\exp(\eta^*) \left( 1 - \exp(-\eta_i)\right) \exp\left(-\sum_{j=i+1}^{t} \eta_j \right) \\
    &=\eta^* \exp(\eta^*)\sum_{i=\tau}^{t} \left[ \exp\left(-\sum_{j=i+1}^{t} \eta_j \right)- \exp\left(-\sum_{j=i}^{t} \eta_j \right)\right] \\
    &= \eta^* \exp(\eta^*)\left[1- \exp\left(-\sum_{j=\tau}^{t} \eta_j \right)\right]\\
    &\leq \eta^* \exp(\eta^*) = \frac{\epsilon}{2},
\end{align*}
\begin{align*}
    \sum_{i=1}^{t} \eta_i^2 \exp\left(-\sum_{j=i+1}^{t} \eta_j \right) &= \sum_{i=1}^{\tau-1} \eta_i^2 \exp\left(-\sum_{j=i+1}^{t} \eta_j \right) + \sum_{i=\tau}^{t} \eta_i^2 \exp\left(-\sum_{j=i+1}^{t} \eta_j \right) \\
    & \leq \frac{\epsilon}{2} + \frac{\epsilon}{2} = \epsilon.
\end{align*}
\end{proof}

\begin{proof}[Proof of \Cref{thm:1/t_convergence}]
From \Cref{lem:convergence_to_ball}, $\norm{\vx_t} \leq \max{\{\norm{\vx_0}, \frac{1}{\lambda}\}} = B$ for $t \geq 0$. Define $C\triangleq \frac{H(1+\lambda B)^2}{2 \lambda^2}\frac{4}{(t+1)^2}$.

    We have that for $t=1$, 
    \begin{align*}
        L(\vx_1) - L(\vx^*) &\leq (1-1) (L(\vx_0) - L(\vx^*)) + \frac{H(1+\lambda B)^2}{2 \lambda^2} \\
        & = C \leq \frac{4C}{3}.
    \end{align*}
    Suppose $L(\vx_{t-1}) - L(\vx^*) \leq \frac{4C}{t+1}$, we have that
    \begin{align*}
        L(\vx_t) - L(\vx^*) &\leq (1-\frac{2}{t+1}) (L(\vx_{t-1}) - L(\vx^*)) + \frac{H(1+\lambda B)^2}{2 \lambda^2}\frac{4}{(t+1)^2} \\
        &\leq \frac{t-1}{t+1} \frac{4C}{t+1} + \frac{4C}{(t+1)^2}\\
        &=\frac{4Ct}{(t+1)^2} \leq \frac{4C}{t+2}.
    \end{align*}
\end{proof}

\subsection{Omitted Proofs for \Cref{lem:nsd_wd_KKT}}\label{sec:proof_nsd_wd_KKT}
\begin{proof}[Proof of \Cref{lem:nsd_wd_KKT}]$ $
    \begin{enumerate}
        \item For any $\eps > 0$, there exists $t'$ such that $\norm{\vx_t - \vx_\infty} \leq \frac{\eps}{2 \lambda}$ for any $t > t'$. Because $\eta_t \vDelta_t = \vx_{t-1} - \vx_{t} - \lambda \eta_t \vx_{t-1}$, we have that
        \begin{align*}
            \frac{\sum_{t=1}^T \eta_t \vDelta_t}{\sum_{t=1}^T \eta_t} &= \frac{\vx_0 - \vx_{T} - \lambda \sum_{t=1}^T \eta_t \vx_{t-1}}{\sum_{t=1}^T \eta_t} \\
            &=\frac{\vx_0 - \vx_{T}-\lambda(\sum_{t=1}^{t'} \eta_t \vx_{t-1}-\sum_{t=1}^{t'} \eta_t \vx_\infty)}{\sum_{t=1}^T \eta_t} -\lambda \frac{\sum_{t=1}^{t'} \eta_t \vx_\infty + \sum_{t=t'+1}^T \eta_t \vx_{t-1}}{\sum_{t=1}^T \eta_t}.
        \end{align*}
        There exists $T' \geq t'$ such that $\sum_{t=1}^T \eta_t \geq \frac{2}{\eps} \left(\norm{\vx_0 -\vx_\infty -\lambda \left(\sum_{t=1}^{t'} \eta_t \vx_{t-1}-\sum_{t=1}^{t'} \eta_{t} \vx_\infty \right)} + \frac{\eps}{2} \right)$ for $T \geq T'$. Then we have
        \begin{align*}
            \norm{\frac{\sum_{t=1}^T \eta_t \vDelta_t}{\sum_{t=1}^T \eta_t} +\lambda \vx_\infty} &\leq \frac{\norm{\vx_0 - \vx_{T}-\lambda \left(\sum_{t=1}^{t'} \eta_t \vx_{t-1}-\sum_{t=1}^{t'} \eta_t \vx_\infty \right)}}{\sum_{t=1}^T \eta_t} + \lambda \frac{\sum_{t=t'+1}^T \eta_t \norm{\vx_{t-1} - \vx_\infty}}{\sum_{t=1}^T \eta_t} \\
            &\leq \frac{\norm{\vx_0 -\vx_\infty -\lambda \left(\sum_{t=1}^{t'} \eta_t \vx_{t-1}-\sum_{t=1}^{t'} \eta_t \vx_\infty \right)} + \norm{\vx_{T}-\vx_\infty}}{\sum_{t=1}^T \eta_t} + \lambda \frac{\eps}{2 \lambda} \\
            &\leq \frac{\eps}{2} + \frac{\eps}{2} = \eps.
        \end{align*}
        So $\vDelta_\infty := \frac{\sum_{t=1}^T \eta_t \vDelta_t}{\sum_{t=1}^T \eta_t}$ exists and $\vDelta_\infty = -\lambda \vx_\infty$. 
        \item Because $\nabla L(\vx)$ is a continuous function and $\lim_{t \rightarrow \infty} \vx_t = \vx_\infty$, $\lim_{t\rightarrow \infty} \nabla L(\vx_t) = \nabla L(\vx_\infty)$. For any $\epsilon>0$, there exists $T_1$ such that \[\left|\norm{\nabla L(\vx_t)}_*-\norm{\nabla L(\vx_\infty)}_* \right| \leq \norm{\nabla L(\vx_t) - \nabla L(\vx_\infty) }_* \leq \frac{\epsilon}{3}\] 
        for any $t \geq T_1$. It also holds that \[|\inner{\nabla L(\vx_t)-\nabla L(\vx_\infty)}{\vDelta_t}| \leq \norm{\nabla L(\vx_t) - \nabla L(\vx_\infty) }_* \norm{\vDelta_t} \leq \frac{\epsilon}{3}\] because $\norm{\vDelta_t}\leq 1$.
        Because $\sum_{t=1}^\infty \eta_t=\infty$, there exists $T_2 \geq T_1$ such that \[\sum_{t=1}^{T_2} \eta_t \geq \frac{3}{\epsilon}\left|\sum_{t=1}^{T_1}\eta_t \left(\norm{\nabla L(\vx_t)}_*-\norm{\nabla L(\vx_\infty)}_* + \inner{\nabla L(\vx_\infty)-\nabla L(\vx_t)}{\vDelta_t}\right) \right|.\] Then for any $T \geq T_2$, we have that
        \begin{align*}
            &\left|\frac{\sum_{t=1}^T \eta_t \inner{\nabla L(\vx_\infty)}{\vDelta_t}}{\sum_{t=1}^T \eta_t} -\norm{\nabla L(\vx_\infty)}_*\right| \\
            =&\left|\frac{\sum_{t=1}^T \eta_t \inner{\nabla L(\vx_t)}{\vDelta_t}}{\sum_{t=1}^T \eta_t} + \frac{\sum_{t=1}^T \eta_t \inner{\nabla L(\vx_\infty)-\nabla L(\vx_t)}{\vDelta_t}}{\sum_{t=1}^T \eta_t}-\norm{\nabla L(\vx_\infty)}_*\right| \\
            =&\left|\frac{\sum_{t=1}^T \eta_t \norm{\nabla L(\vx_t)}_*}{\sum_{t=1}^T \eta_t} + \frac{\sum_{t=1}^T \eta_t \inner{\nabla L(\vx_\infty)-\nabla L(\vx_t)}{\vDelta_t}}{\sum_{t=1}^T \eta_t}-\norm{\nabla L(\vx_\infty)}_*\right|\\
            =&\left|\frac{\sum_{t=1}^T \eta_t (\norm{\nabla L(\vx_t)}_*-\norm{\nabla L(\vx_\infty)}_*)}{\sum_{t=1}^T \eta_t} + \frac{\sum_{t=1}^T \eta_t \inner{\nabla L(\vx_\infty)-\nabla L(\vx_t)}{\vDelta_t}}{\sum_{t=1}^T \eta_t}\right|\\
            \leq&\frac{\left|\sum_{t=1}^{T_1}\eta_t \left(\norm{\nabla L(\vx_t)}_*-\norm{\nabla L(\vx_\infty)}_* + \inner{\nabla L(\vx_\infty)-\nabla L(\vx_t)}{\vDelta_t}\right) \right|}{\sum_{t=1}^T \eta_t} \\
            +& \frac{\sum_{t=T_1+1}^T \eta_t \left|\norm{\nabla L(\vx_t)}_*-\norm{\nabla L(\vx_\infty)}_* \right|}{\sum_{t=1}^T \eta_t}+ \frac{\sum_{t=T_1+1}^T \eta_t \left|\inner{\nabla L(\vx_\infty)-\nabla L(\vx_t)}{\vDelta_t}\right|}{\sum_{t=1}^T \eta_t}\\
            \leq & \frac{\epsilon}{3} + \frac{\epsilon}{3}\frac{\sum_{t=T_1+1}^T \eta_t}{\sum_{t=1}^T \eta_t}+\frac{\epsilon}{3}\frac{\sum_{t=T_1+1}^T \eta_t}{\sum_{t=1}^T \eta_t} \\
            \leq &\epsilon.
        \end{align*}
        Therefore, we prove that $\norm{\nabla L(\vx_\infty)}_* = \lim_{T \to \infty} \frac{\sum_{t=1}^T \eta_t \inner{\nabla L(\vx_\infty)}{\vDelta_t}}{\sum_{t=1}^T \eta_t}$. On the other hand, we have that
        \begin{align*}
            \inner{\nabla L(\vx_\infty)}{\vDelta_\infty} &= \inner{\nabla L(\vx_\infty)}{\lim_{T \to \infty} \frac{\sum_{t=1}^T \eta_t \vDelta_t}{\sum_{t=1}^T \eta_t}}\\
            &=\lim_{T \to \infty} \inner{\nabla L(\vx_\infty)}{ \frac{\sum_{t=1}^T \eta_t \vDelta_t}{\sum_{t=1}^T \eta_t}}\\
            &=\lim_{T \to \infty} \frac{\sum_{t=1}^T \eta_t \inner{\nabla L(\vx_\infty)}{\vDelta_t}}{\sum_{t=1}^T \eta_t}, 
        \end{align*}
        which finishes the proof. 
        \item For any $T$, we know $\norm{\frac{\sum_{t=1}^T \eta_t \vDelta_t}{\sum_{t=1}^T \eta_t}} \leq \frac{\sum_{t=1}^T \eta_t \norm{\vDelta_t}}{\sum_{t=1}^T \eta_t} \leq \frac{\sum_{t=1}^T \eta_t}{\sum_{t=1}^T \eta_t}=1$. By the continuity of $\norm{\cdot}$, $\norm{\Delta_\infty} = \lim_{T \rightarrow\infty} \norm{\frac{\sum_{t=1}^T \eta_t \vDelta_t}{\sum_{t=1}^T \eta_t}} \leq 1$. 
    \end{enumerate}
\end{proof}
\section{Omitted Proofs in \Cref{sec:adamw}}\label{sec:adamw_details}

\subsection{Omitted proofs for upper bound for average update size of \adam}\label{sec:proof_amortized_update_bound}

\begin{proof}[Proof of \Cref{lem:amortized_update_bound}]
    We first represent $m_t$ and $v_t$ as a weighted sum of $g_t$ and $g_t^2$. 
    \begin{align*}
        m_t &= \beta_1 m_{t-1} + (1-\beta_1) g_t = \beta_1^t m_0+(1-\beta_1)\sum_{i=0}^{t-1} \beta_1^i g_{t-i},\\
        v_t &\geq \beta_2 v_{t-1} + (1-\beta_2) g_t^2 \geq \beta_2^t v_0+(1-\beta_2)\sum_{i=0}^{t-1} \beta_2^i g_{t-i}^2.
    \end{align*}By Cauchy–Schwarz inequality, we have that
    \begin{align*}
        \left|\sum_{t=1}^T \eta_t \Delta_t\right| &= \left| \sum_{t=1}^T \eta_t \frac{\beta_1^t m_0+(1-\beta_1)\sum_{i=0}^{t-1} \beta_1^i g_{t-i}}{\sqrt{v_t}} \right|\\
        &\leq \left[ \sum_{t=1}^T \eta_t \left(\frac{\beta_1^t m_0^2}{v_t}+\sum_{i=0}^{t-1} (1-\beta_1)\beta_1^i \frac{g_{t-i}^2}{v_t} \right)\right]^{\frac{1}{2}} \left[ \sum_{t=1}^T \eta_t \left( \beta_1^t + \sum_{i=0}^{t-1} (1-\beta_1)\beta_1^i \right)\right]^{\frac{1}{2}}\\
        % &\leq \left(\sum_{t=1}^T \eta_t \frac{\beta_1^t m_0^2}{v_t}+ \sum_{t=1}^T \eta_t \sum_{i=0}^{t-1} (1-\beta_1)\beta_1^i \frac{v_{t-i} - \beta_2 v_{t-i-1}}{(1-\beta_2)v_t} \right)^{\frac{1}{2}} \left( \sum_{t=1}^T \eta_t  \right)^{\frac{1}{2}} \\
        &\leq \left(\sum_{t=1}^T \eta_t \frac{\beta_1^t v_0}{v_t} + \sum_{t=1}^T \eta_t \sum_{i=0}^{t-1} (1-\beta_1)\beta_1^i \frac{v_{t-i} - \beta_2 v_{t-i-1}}{(1-\beta_2)v_t} \right)^{\frac{1}{2}} \left( \sum_{t=1}^T \eta_t \right)^{\frac{1}{2}}.
    \end{align*}
    We further analyze the first time in RHS and have that
    \begin{align*}
        &\sum_{t=1}^T \eta_t \frac{\beta_1^t v_0}{v_t} +\sum_{t=1}^T \eta_t \sum_{i=0}^{t-1} (1-\beta_1)\beta_1^i \frac{v_{t-i} - \beta_2 v_{t-i-1}}{(1-\beta_2)v_t} \\
        =& \sum_{t=1}^T \eta_t \frac{\beta_1^t v_0}{v_t} +\sum_{t=1}^T \eta_t \frac{1-\beta_1}{(1-\beta_2) v_t} \left(v_t-\beta_1^{t-1} \beta_2 v_0+\sum_{i=1}^{t-1} (\beta_1^i - \beta_2 \beta_1^{i-1}) v_{t-i}\right)\\
        =&\sum_{t=1}^T \eta_t + \frac{\beta_2 - \beta_1}{1-\beta_2} \sum_{t=1}^T \eta_t \left(1-\beta_1^{t-1}\frac{v_0}{v_t}-(1-\beta_1)\sum_{i=1}^{t-1} \beta_1^{i-1} \frac{v_{t-i}}{v_t}\right)\\
        \leq &\sum_{t=1}^T \eta_t + \frac{\beta_2 - \beta_1}{1-\beta_2} \sum_{t=1}^T \eta_t \left(\beta_1^{t-1}+(1-\beta_1)\sum_{i=1}^{t-1} \beta_1^{i-1} (1-\frac{v_{t-i}}{v_t})\right) \\
         \leq& \sum_{t=1}^T \eta_t + \frac{\beta_2 - \beta_1}{1-\beta_2} \sum_{t=1}^T \eta_t \left(\beta_1^{t-1}+(1-\beta_1)\sum_{i=1}^{t-1} \beta_1^{i-1} \ln{(\frac{v_{t}}{v_{t-i}})}\right)\\
         =&\sum_{t=1}^T \eta_t +\frac{\beta_2-\beta_1}{1-\beta_2}\sum_{t=1}^T \eta_t \beta_1^{t-1}+ \frac{(\beta_2-\beta_1)(1-\beta_1)}{1-\beta_2}
        \sum_{t=1}^T \left(\eta_t\frac{1-\beta_1^{t-1}}{1-\beta_1} -\sum_{i=1}^{T-t} \eta_{t+i} \beta_1^{i-1}\right)\ln{v_t} \\
        =&\sum_{t=1}^T \eta_t +\frac{\beta_2-\beta_1}{1-\beta_2}\sum_{t=1}^T \eta_t \beta_1^{t-1}+ \frac{(\beta_2-\beta_1)(1-\beta_1)}{1-\beta_2}
        \sum_{t=2}^T \left(\eta_t\frac{1-\beta_1^{t-1}}{1-\beta_1} -\sum_{i=1}^{T-t} \eta_{t+i} \beta_1^{i-1}\right)\ln{\frac{v_t}{v_1}}. 
    \end{align*}
    When $\beta_1 = \beta_2$, we have that
    \begin{align*}
        \left|\Delta_t \right| &= \left| \frac{\beta_1^t m_0+(1-\beta_1)\sum_{i=0}^{t-1} \beta_1^i g_{t-i}}{\sqrt{v_t}} \right| \\
        &\leq \left( \frac{\beta_1^t m_0^2 + (1-\beta_1)\sum_{i=0}^{t-1} \beta_1^i g_{t-i}^2}{v_t}\right)^{\frac{1}{2}}\left( \beta_1^t + \sum_{i=0}^{t-1} (1-\beta_1)\beta_1^i \right)^{\frac{1}{2}}\\
        &\leq \left(\frac{\beta_1^t v_0 + (1-\beta_1)\sum_{i=0}^{t-1} \beta_1^i g_{t-i}^2}{v_t} \right)^{\frac{1}{2}}\\
        &=1.
    \end{align*}
\end{proof}
\subsection{Proof for \Cref{lem:property_converged_point}}
\begin{proof}[Proof of \Cref{lem:property_converged_point}]$ $

    \begin{enumerate}
    \item The proof for this part is the same as the proof of \Cref{lem:nsd_wd_KKT} in \Cref{sec:proof_nsd_wd_KKT}.
        \item If $\nabla L(\vx_\infty) = \vec{0}$, $\inner{\nabla L(\vx_\infty)}{\vDelta_\infty} =0= \norm{\nabla L(\vx_\infty)}_1$. 
        
        If $\nabla L(\vx_\infty) \neq \vec{0}$, we consider each coordinate $j$ such that $\nabla L(\vx_\infty)_j \neq 0$. Since we have that 
        $\lim_{t \rightarrow \infty} \nabla L(\vx_t)_j = \nabla L(\vx_\infty)_j$, we can get the convergence for $\vm_{t,j}$ and $\vv_{t,j}$. 
        \begin{align*}
            \vm_{t,j}&= (1-\beta_1) \sum_{i=0}^t \beta_1^i \vg_{t-i,j}\rightarrow \nabla L(\vx_\infty)_j, \\
            \vv_{t,j}&= (1-\beta_2) \sum_{i=0}^t \beta_2^i \vg_{t-i,j}^2 \rightarrow \nabla L(\vx_\infty)_j^2 \neq 0.
        \end{align*}
        Then we have that
        $\lim_{t \rightarrow \infty} \vDelta_{t,j} =\lim_{t \rightarrow \infty} \frac{\vm_{t,j}}{\sqrt{\vv_{t,j}}}= \text{sign}(\nabla L(\vx_\infty)_j)$. 
        For any $\eps > 0$, there exists $t'$ such that $\norm{\vDelta_{t,j} - \text{sign} \left(\nabla L(\vx_\infty)_j \right)} \leq \frac{\eps}{2}$ for $t \geq t'$. And there exists $T' \geq t'$ such that $\sum_{t=1}^T \eta_t \geq \frac{2}{\eps} \sum_{t=1}^{t'} \eta_t \left(\vDelta_{t,j}-\text{sign}\left(\nabla L(\vx_\infty)_j\right)\right)$ for any $T \geq T'$. Then for any $T \geq T'$, we have that
        \begin{align*}
            &\norm{\frac{\sum_{t=1}^T \eta_t \vDelta_{t,j}}{\sum_{t=1}^T \eta_t} - \text{sign}(\nabla L(\vx_\infty)_j)} \\
            \leq& \norm{\frac{\sum_{t=1}^{t'} \eta_t \left(\vDelta_{t,j}-\text{sign}(\nabla L(\vx_\infty)_j)\right)}{\sum_{t=1}^T \eta_t}} + \frac{\sum_{t=t'+1}^T \eta_t \norm{\vDelta_{t,j} - \text{sign}(\nabla L(\vx_\infty)_j)}}{\sum_{t=1}^T \eta_t} \\
            \leq& \frac{\eps}{2} + \frac{\eps}{2} = \eps. 
        \end{align*}
        So $\vDelta_{\infty,j} = \lim_{T \rightarrow \infty} \frac{\sum_{t=1}^T \eta_t \vDelta_{t,j}}{\sum_{t=1}^T \eta_t}=\text{sign}(\nabla L(\vx_\infty)_j)$ for $\nabla L(\vx_\infty)_j \neq 0$. Then we have that
        \begin{align*}
            \inner{\nabla L(\vx_\infty)}{\vDelta_\infty} &= \sum_{\nabla L(\vx_\infty)_j\neq 0}\nabla L(\vx_\infty)_j \vDelta_{\infty,j} \\
            &= \sum_{\nabla L(\vx_\infty)_j\neq 0}|\nabla L(\vx_\infty)_j|  \\
            &= \norm{\nabla L(\vx_\infty)}_1. 
        \end{align*}
        \item For nonzero coordinate $j$ of $\nabla L(\vx_\infty)$, from above we have $|\vDelta_{\infty,j}| = |\text{sign} (\nabla L(\vx_\infty)_j)|=1$.

        For $j$ such that $\nabla L(\vx_\infty)_j=0$, we know $\lim_{t \rightarrow \infty} \vg_{t,j} = \lim_{t \rightarrow \infty} \vm_{t,j} = \lim_{t \rightarrow \infty} \vv_{t,j} = 0$. We employ the upper bound for average update in \Cref{lem:amortized_update_bound} since $\{\vg_{t,j}\}_{t=1}^\infty$ and $\{\vv_{t,j}\}_{t=0}^\infty$ in \Cref{alg:adamw} satisfy the condition that $\vv_{t,j} - \beta_2 \vv_{t-1,j} \geq (1-\beta_2)\vg_{t,j}^2$ and $\vm_{0,j} = 0 \leq \sqrt{\vv_{0,j}}$. 
        By \Cref{lem:amortized_update_bound} we have 
        \begin{align*}
            & \left|\frac{\sum_{t=1}^T \eta_t \vDelta_{t,j}}{\sum_{t=1}^T \eta_t}\right|\\
            \leq & \left(\frac{\sum_{t=1}^T \eta_t +\frac{\beta_2-\beta_1}{1-\beta_2}\sum_{t=1}^T \eta_t \beta_1^{t-1}+ \frac{(\beta_2-\beta_1)(1-\beta_1)}{1-\beta_2}
        \sum_{t=2}^T \left(\eta_t\frac{1-\beta_1^{t-1}}{1-\beta_1} -\sum_{i=1}^{T-t} \eta_{t+i} \beta_1^{i-1}\right)\ln{\frac{\vv_{t,j}}{\vv_{1,j}}}}{\sum_{t=1}^T \eta_t}\right)^{\frac{1}{2}}. 
        \end{align*}
        The denominator goes to $\infty$ when $T \rightarrow \infty$. So it suffices to bound the last two terms in the numerator by constants in order to show $\norm{\vDelta_\infty}\le 1$.
        Because $\eta_t$ is non-increasing in $t$, it holds that 
        \begin{align*}
            \sum_{t=1}^T\eta_t \beta_1^t \leq \sum_{t=1}^T \eta_1 \beta_1^t \leq \frac{\eta_1 \beta_1}{1-\beta_1}.
        \end{align*}
        For the last term, we first analyze the coefficient between each $\ln{\vv_{t,j}}$. 
        Define $\alpha_t = \eta_t \frac{1-\beta_1^{t-1}}{1-\beta_1} -\sum_{i=1}^{T-t} \eta_{t+i} \beta_1^{i-1}$. We claim that $\abs{\alpha_t} \leq \max{\{\frac{\beta_1^{t-1}}{1-\beta_1}\eta_{t+1}, \frac{\eta_t}{1-\beta_1} \}} = \frac{\eta_t}{1-\beta_1}$. This is because 
        \begin{align*}
            \alpha_t&\leq \eta_t \frac{1-\beta_1^{t-1}}{1-\beta_1} \leq \frac{\eta_t}{1-\beta_1},
        \end{align*}
        and again by monotonicity of learning rates $\eta_t$, we have that
        \begin{align*}
            \alpha_t\geq \eta_{t+1} \frac{1-\beta_1^{t-1}}{1-\beta_1} -\sum_{i=1}^{T-t} \eta_{t+1} \beta_1^{i-1} \geq \eta_{t+1}\left( \frac{1-\beta_1^{t-1}}{1-\beta_1} - \sum_{i=1}^\infty \beta_1^{i-1}\right) 
            =-\frac{\beta_1^{t-1}}{1-\beta_1}\eta_{t+1}.
        \end{align*}
        
        We can also have $\ln{\frac{\vv_{t,j}}{\vv_{1,j}}} \geq (t-1) \ln{\beta_2}$ because \[\ln{\vv_{t,j}}-\ln{\vv_{t-1,j}} = \ln{\frac{\vv_{t,j}}{\vv_{t-1,j}}}=\ln{\frac{\beta_2 \vv_{t-1,j}+(1-\beta_2) \vg_{t,j}^2}{\vv_{t-1,j}}} \geq \ln{\beta_2}.\] And there exists $t'$ such that $\ln{\frac{\vv_{t,j}}{\vv_{1,j}}} \leq 0$ for any $t \geq t'$ because $\lim_{t \rightarrow \infty} \vv_{t,j}=0$.
        Then 
        \begin{align*}
            \sum_{t=2}^T \alpha_t\ln{\frac{\vv_{t,j}}{\vv_{1,j}}}  &\leq \sum_{t=2}^{t'} \abs{\alpha_t} \abs{\ln{\frac{\vv_{t,j}}{\vv_{1,j}}}}  + \sum_{t=t'+1}^T \alpha_t\ln{\frac{\vv_{t,j}}{\vv_{1,j}}} \\
            &\leq \sum_{t=2}^{t'}\frac{\eta_t}{1-\beta_1}\abs{\ln{\frac{\vv_{t,j}}{\vv_{1,j}}}}  + \sum_{\alpha_t \leq 0,\ t'<t \leq T}\alpha_t \ln{\frac{\vv_{t,j}}{\vv_{1,j}}} \\
            &\leq \frac{\sum_{t=2}^{t'}\eta_t\abs{\ln{\frac{\vv_{t,j}}{\vv_{1,j}}}}}{1-\beta_1}+\sum_{\alpha_t \leq 0,\ t'<t \leq T} \left(-\frac{\beta_1^{t}}{1-\beta_1}\eta_{t+1}\right) (t-1) \ln{\beta_2}\\
            &\leq \frac{\sum_{t=2}^{t'}\eta_t\abs{\ln{\frac{\vv_{t,j}}{\vv_{1,j}}}}}{1-\beta_1} +(-\ln{\beta_2})\sum_{t=1}^T \frac{(t-1)\beta_1^t}{1-\beta_1} \eta_{t+1} \\
            &\leq \frac{\sum_{t=2}^{t'}\eta_t\abs{\ln{\frac{\vv_{t,j}}{\vv_{1,j}}}}}{1-\beta_1}+(-\ln{\beta_2}) \eta_1 \sum_{t=1}^T \frac{(t-1)\beta_1^t}{1-\beta_1} \\
            &\leq \frac{\sum_{t=2}^{t'}\eta_t\abs{\ln{\frac{\vv_{t,j}}{\vv_{1,j}}}}}{1-\beta_1}-\frac{\eta_1\beta_1^2\ln{\beta_2}}{(1-\beta_1)^3}.
        \end{align*}
        Define $C:=\frac{(\beta_2-\beta_1)\eta_1\beta_1}{(1-\beta_2)(1-\beta_1)} + \frac{\beta_2-\beta_1}{1-\beta_2}\left(\sum_{t=2}^{t'}\eta_t\abs{\ln{\vv_{t,j}}}-\frac{\eta_1\beta_1^2\ln{\beta_2}}{(1-\beta_1)^2}\right)$, we now have 
        \begin{equation*}
            \left| \frac{\sum_{t=1}^T \eta_t \vDelta_{t,j}}{\sum_{t=1}^T \eta_t} \right|%&\leq \left( \frac{\sum_{t=0}^T \eta_t \sum_{i=0}^t (1-\beta_1)\beta_1^i \frac{v_{t-i} - \beta_2 v_{t-i-1}}{(1-\beta_2)v_t}}{\sum_{t=0}^T \eta_t}\right)^{\frac{1}{2}}\\
            \leq \left( \frac{\sum_{t=1}^T \eta_t +C}{\sum_{t=1}^T \eta_t}\right)^{\frac{1}{2}}. 
        \end{equation*}
        Then $\left| \vDelta_{\infty, j}\right| =  \left|\lim\limits_{T\to\infty} \frac{\sum_{t=1}^T \eta_t \vDelta_{t,j}}{\sum_{t=1}^T \eta_t} \right| \le \lim\limits_{T\to\infty} \abs{ \frac{\sum_{t=1}^T \eta_t \vDelta_{t,j}}{\sum_{t=1}^T \eta_t}}\leq 1$ because $\sum_{t=1}^\infty \eta_T=\infty$. This completes the proof. %$\norm{\vDelta}_\infty \leq 1$
    \end{enumerate}
\end{proof}

\subsection{A counter example when $\beta_1 > \beta_2$}\label{sec:counter_example}
For any $\lambda$ and $\beta_1 > \beta_2$, we provide the following example for which the iterates of \adamw will converge and the $\ell_\infty$ norm of the converged solution is larger than $\frac{1}{\lambda}$.

For some sufficiently small $\eta$, denote $\Tilde{x} = -\frac{1}{\lambda}\frac{\frac{1-\beta_1}{1-\lambda\eta-\beta_1}}{\sqrt{\frac{1-\beta_2}{(1-\lambda\eta)^2-\beta_2}}}$. $L(x)$ is defined as $\frac{1}{2}(x-\Tilde{x})^2$. For any starting point $x_0>\Tilde{x}$, $m_0$ is set as $\frac{1-\beta_1}{1-\lambda\eta-\beta_1}g_1$ and $v_0$ is set as $\frac{1-\beta_2}{(1-\lambda\eta)^2-\beta_2}g_1^2$ with $g_1=\nabla L(x_0) = x_0-\Tilde{x}$. We show by induction that $\frac{m_t}{\sqrt{v_t}} = -\lambda \Tilde{x}$ and $x_{t}-\Tilde{x} = (1-\lambda \eta)(x_{t-1}-\Tilde{x})$ for any $t \geq 1$. 

When $t=1$, we have that 
\begin{align*}
    m_1&=\beta_1 m_0 +(1-\beta_1)g_1=\frac{(1-\beta_1)(1-\lambda\eta)}{1-\lambda\eta-\beta_1}g_1, \\
    v_1&=\beta_2 v_0 + (1-\beta_2)g_1^2=\frac{(1-\beta_2)(1-\lambda\eta)^2}{(1-\lambda\eta)^2-\beta_2}g_1^2.
\end{align*}
Then $\frac{m_1}{\sqrt{v_1}} = \text{sign}(g_1)\frac{\frac{1-\beta_1}{1-\lambda\eta-\beta_1}}{\sqrt{\frac{1-\beta_2}{(1-\lambda\eta)^2-\beta_2}}} = \frac{\frac{1-\beta_1}{1-\lambda\eta-\beta_1}}{\sqrt{\frac{1-\beta_2}{(1-\lambda\eta)^2-\beta_2}}} = -\lambda \Tilde{x}$, which proves the first claim. For the second claim, we have that 
\begin{equation*}
    x_1-\Tilde{x} = x_0-\eta \frac{m_1}{\sqrt{v_1}}-\lambda\eta x_0-\Tilde{x}=x_0+\lambda\eta\Tilde{x}-\lambda\eta x_0-\Tilde{x}=(1-\lambda\eta)(x_0-\Tilde{x}).
\end{equation*}

Suppose the claims hold for any $0\leq i<t$. Then $g_{i+1}=\nabla L(x_i) = x_i-\Tilde{x} = (1-\lambda\eta)^i g_1$. We have that 
\begin{align*}
    m_t&=\beta_1^t m_0 + (1-\beta_1)\sum_{i=1}^t \beta_1^{t-i} g_i \\
    &=\beta_1^t \frac{1-\beta_1}{1-\lambda\eta-\beta_1}g_1 + (1-\beta_1)\sum_{i=1}^t \beta_1^{t-i} (1-\lambda\eta)^{i-1}g_1 \\
    &=\left(\beta_1^t \frac{1-\beta_1}{1-\lambda\eta-\beta_1} + (1-\beta_1) \frac{(1-\lambda\eta)^t-\beta_1^t}{1-\lambda\eta-\beta_1}\right)g_1 \\
    &=\frac{1-\beta_1}{1-\lambda\eta-\beta_1}(1-\lambda\eta)^tg_1, 
\end{align*}
and 
\begin{align*}
    v_t &= \beta_2^t v_0 + (1-\beta_2)\sum_{i=1}^t \beta_2^{t-i} g_i^2 \\
    &=\beta_2^t \frac{1-\beta_2}{(1-\lambda\eta)^2-\beta_2}g_1^2 + (1-\beta_2)\sum_{i=1}^t \beta_2^{t-i} (1-\lambda\eta)^{2(i-1)}g_1^2 \\
    &=\left(\beta_2^t \frac{1-\beta_2}{(1-\lambda\eta)^2-\beta_2} + (1-\beta_2) \frac{(1-\lambda\eta)^{2t}-\beta_2^t}{(1-\lambda\eta)^2-\beta_2}\right)g_1^2 \\
    &=\frac{1-\beta_2}{(1-\lambda\eta)^2-\beta_2}(1-\lambda\eta)^{2t}g_1^2. 
\end{align*}
Then $\frac{m_t}{\sqrt{v_t}} = \frac{\frac{1-\beta_1}{1-\lambda\eta-\beta_1}}{\sqrt{\frac{1-\beta_2}{(1-\lambda\eta)^2-\beta_2}}} = -\lambda \Tilde{x}$. We also have that 
\begin{equation*}
    x_{t}-\Tilde{x} = x_{t-1}-\eta \frac{m_t}{\sqrt{v_t}}-\lambda\eta x_{t-1}-\Tilde{x}=x_0+\lambda\eta\Tilde{x}-\lambda\eta x_0-\Tilde{x}=(1-\lambda\eta)(x_0-\Tilde{x}).
\end{equation*}. 

In this regime, $x_t$ will converge to $\Tilde{x}$ because $|x_t-\Tilde{x}|=O((1-\lambda\eta)^t)$. However, when $\beta_1 > \beta_2$ and $\lambda\eta$ is very small, $|\Tilde{x}|$ can be larger than $\frac{1}{\lambda}$. For example, when $\beta_1=0.99$, $\beta_2=0.9$, $\lambda=0.1$ and $\eta=0.01$, $|\Tilde{x}|=10.999>\frac{1}{\lambda}$. 

\subsection{Proof for upper bound for norm of iterates in \adamw}\label{sec:iterate_norm_bound}
\begin{proof}[Proof of \Cref{lem:iterate_norm_bound}]

For \adamw with constant learning rate $\eta$ and each coordinate $j$, $\vx_{T,j}$ can be written as weighted average of past update
\begin{align*}
    \vx_{T,j}&=(1-\lambda \eta)^T \vx_{0,j} + \sum_{t=0}^{T-1} \eta(1-\lambda \eta)^t \frac{\vm_{T-t,j}}{\sqrt{\vv_{T-t,j}}} \\
    &= (1-\lambda \eta)^T \vx_{0,j} +\sum_{t=1}^T \eta(1-\lambda \eta)^{T-t} \frac{\vm_{t,j}}{\sqrt{\vv_{t,j}}}.
\end{align*}
Define $\eta_t = \eta(1-\lambda \eta)^{T-t}$ for $1 \leq t \leq T$. We apply \Cref{lem:amortized_update_bound} on $\{\vv_{t,j}\}_{t=1}^T$ and $\{\vg_{t,j}\}_{t=1}^T$ to bound $\left| \sum_{t=1}^T \eta(1-\lambda \eta)^{T-t} \frac{\vm_{t,j}}{\sqrt{\vv_{t,j}}}\right|$. 

We first compute $\sum_{t=1}^T \eta_t = \frac{1-(1-\lambda \eta)^T}{\lambda} \leq \frac{1}{\lambda}$. For the second term in \Cref{eq:amortized_bound}, we have that
\begin{align*}
    \frac{\beta_2-\beta_1}{1-\beta_2} \sum_{t=1}^T \eta_t \beta_1^{t-1} &=\frac{\beta_2-\beta_1}{1-\beta_2} \sum_{t=1}^T \eta(1-\lambda \eta)^{T-t} \beta_1^{t-1} \\
    &=\frac{1}{\lambda}\frac{(\beta_2-\beta_1) \lambda \eta[(1-\lambda \eta)^T - \beta_1^T]}{(1-\beta_2)(1-\lambda\eta-\beta_1)}\\
    &\leq \frac{(\beta_2-\beta_1)\eta}{(1-\beta_2)\abs{1-\lambda\eta-\beta_1}}\left[\beta_1^{T} + (1-\lambda\eta)^{T} \right]. 
\end{align*}
For the last term, we define $\alpha_t = \eta_t\frac{1-\beta_1^{t-1}}{1-\beta_1} -\sum_{i=1}^{T-t} \eta_{t+i} \beta_1^{i-1}$ and we can compute the exact form of $\alpha_t$ as following
\begin{align*}
    \alpha_t&=\eta(1-\lambda \eta)^{T-t} \frac{1-\beta_1^{t-1}}{1-\beta_1} -\sum_{i=1}^{T-t} \eta(1-\lambda \eta)^{T-t-i} \beta_1^{i-1}\\
    &=\eta(1-\lambda \eta)^{T-t} \frac{1-\beta_1^{t-1}}{1-\beta_1}-\frac{\eta[(1-\lambda \eta)^{T-t}-\beta_1^{T-t}]}{1-\lambda \eta -\beta_1} \\
    &=\frac{\eta(1-\lambda\eta)^{T-t}(-\lambda\eta)}{(1-\beta_1)(1-\lambda\eta-\beta_1)} -\frac{\eta(1-\lambda\eta)^{T-t} \beta_1^{t-1}}{1-\beta_1} + \frac{\eta\beta_1^{T-t}}{1-\lambda\eta-\beta_1}.
\end{align*}
Then we can bound the last term by showing that
\begin{align*}
    &\frac{(\beta_2-\beta_1)(1-\beta_1)}{1-\beta_2} \sum_{t=2}^T \alpha_t \ln{\frac{\vv_{t,j}}{\vv_{1,j}}}\\
    \leq&\frac{(\beta_2-\beta_1)(1-\beta_1)}{1-\beta_2} \sum_{t=2}^T |\alpha_t| \left|\ln{\frac{\vv_{t,j}}{\vv_{1,j}}}\right|\\
    \leq& C \frac{(\beta_2-\beta_1)(1-\beta_1)}{1-\beta_2} \left[\sum_{t=2}^T \frac{\lambda \eta^2 (1-\lambda \eta)^{T-t}}{(1-\beta_1)|1-\lambda\eta-\beta_1|} + \sum_{t=2}^T \frac{\eta(1-\lambda\eta)^{T-t} \beta_1^{t-1}}{1-\beta_1} + \sum_{t=2}^T \frac{\eta\beta_1^{T-t}}{|1-\lambda\eta-\beta_1|}\right]\\
    =&C \frac{(\beta_2-\beta_1)(1-\beta_1)}{1-\beta_2} \left[\frac{\eta[1-(1-\lambda\eta)^{T-1}]}{(1-\beta_1)|1-\lambda\eta-\beta_1|} + \frac{\beta_1 \eta|(1-\lambda\eta)^{T-1}-\beta_1^{T-1}|}{(1-\beta_1)|1-\lambda\eta-\beta_1|} + \frac{\eta(1-\beta_1^{T-1})}{|1-\lambda\eta-\beta_1|(1-\beta_1)}\right]\\
    \leq&C \frac{(\beta_2-\beta_1)(1-\beta_1)}{1-\beta_2} \left[\frac{\eta[1-(1-\lambda\eta)^{T-1}]}{(1-\beta_1)|1-\lambda\eta-\beta_1|} + \frac{\beta_1 \eta\left((1-\lambda\eta)^{T-1}+\beta_1^{T-1}\right)}{(1-\beta_1)|1-\lambda\eta-\beta_1|} + \frac{\eta(1-\beta_1^{T-1})}{|1-\lambda\eta-\beta_1|(1-\beta_1)}\right]\\
    =&\frac{2C\eta(\beta_2-\beta_1)}{(1-\beta_2)|1-\lambda\eta-\beta_1|} + \frac{C\eta(\beta_2-\beta_1)(\beta_1-1)}{(1-\beta_2)|1-\lambda\eta-\beta_1|}\left[\beta_1^{T-1} + (1-\lambda\eta)^{T-1} \right] \\
    \leq& \frac{2C\eta(\beta_2-\beta_1)}{(1-\beta_2)|1-\lambda\eta-\beta_1|}.
\end{align*}

Therefore, we have the following bound
\begin{align*}
    \lambda|\vx_{T,j}| &\leq \lambda(1-\lambda\eta)^T |\vx_{0,j}| + \lambda\left| \sum_{t=1}^T \eta_t \frac{\vm_{t,j}}{\sqrt{\vv_{t,j}}}\right| \\
    &\leq \lambda(1-\lambda\eta)^T |\vx_{0,j}| + \lambda \sum_{t=1}^T \eta_t \biggl[1+ \frac{(\beta_2-\beta_1)\eta\left[\beta_1^{T} + (1-\lambda\eta)^{T} \right]}{\sum_{t=1}^T \eta_t(1-\beta_2)\abs{1-\lambda\eta-\beta_1}}+ \frac{2C\eta(\beta_2-\beta_1)}{\sum_{t=1}^T \eta_t(1-\beta_2)|1-\lambda\eta-\beta_1|} \biggr]^{\frac{1}{2}} \\
    &\leq \lambda(1-\lambda\eta)^T |\vx_{0,j}| +\left[(\lambda\sum_{t=1}^T \eta_t)^2+ \lambda^2 \sum_{t=1}^T \eta_t\left[ \frac{(\beta_2-\beta_1)\eta\left[\beta_1^{T} + (1-\lambda\eta)^{T} \right]+2C\eta(\beta_2-\beta_1)}{(1-\beta_2)\abs{1-\lambda\eta-\beta_1}}\right]\right]^{\frac{1}{2}}\\%+\frac{2C\eta(\beta_2-\beta_1)}{(1-\beta_2)|1-\lambda\eta-\beta_1|}\right]\right]^{\frac{1}{2}}\\
    &\leq \lambda(1-\lambda\eta)^T |\vx_{0,j}| +\left[1+ \lambda  \frac{(\beta_2-\beta_1)\eta\left[\beta_1^{T} + (1-\lambda\eta)^{T} \right]+2C\eta(\beta_2-\beta_1)}{(1-\beta_2)\abs{1-\lambda\eta-\beta_1}}\right]^{\frac{1}{2}}\\
    &\leq \lambda(1-\lambda\eta)^T |\vx_{0,j}| +\left[1+ \frac{\lambda \eta(\beta_2-\beta_1)\left[\beta_1^{T} + (1-\lambda\eta)^{T} \right]}{2(1-\beta_2)|1-\lambda\eta-\beta_1|}+ C\frac{\lambda \eta(\beta_2-\beta_1)}{(1-\beta_2)|1-\lambda\eta-\beta_1|} \right]\\
    &\leq 1+ \lambda(1-\lambda\eta)^T \norm{\vx_{0}}_\infty +\frac{\lambda \eta(\beta_2-\beta_1)\left[\beta_1^{T} + (1-\lambda\eta)^{T} \right]}{2(1-\beta_2)|1-\lambda\eta-\beta_1|}+ C\frac{\lambda \eta(\beta_2-\beta_1)}{(1-\beta_2)|1-\lambda\eta-\beta_1|}.
\end{align*}
This completes the proof. 
\end{proof}
\section{Experimental Details and More Results}\label{sec:exp_details}

The architecture of the two-layer transformer is the same as in~\citet{kunstner2023noise}, which is also used as a tutorial example in PyTorch. It consists of a 200-dimensional embedding layer, $2$ transformer layers and a linear layer. Each transformer layer consists of a $2$-head self-attention and an MLP with a hidden dimension $200$. The experiments are run on a single A4000 or A6000. 

As mentioned in \Cref{sec:experiments}, we present the results for another random seed in \Cref{fig:ptb_seed_1}. We also plot the results in full range for the two random seeds in \Cref{fig:ptb_seed_0_full} and \Cref{fig:ptb_seed_1_full} to show that the $\ell_\infty$ norm of parameters for \adam keeps increasing. 
\begin{figure}[h]
    \centering
    \begin{subfigure}[b]{0.24\textwidth}
        \includegraphics[width=\textwidth]{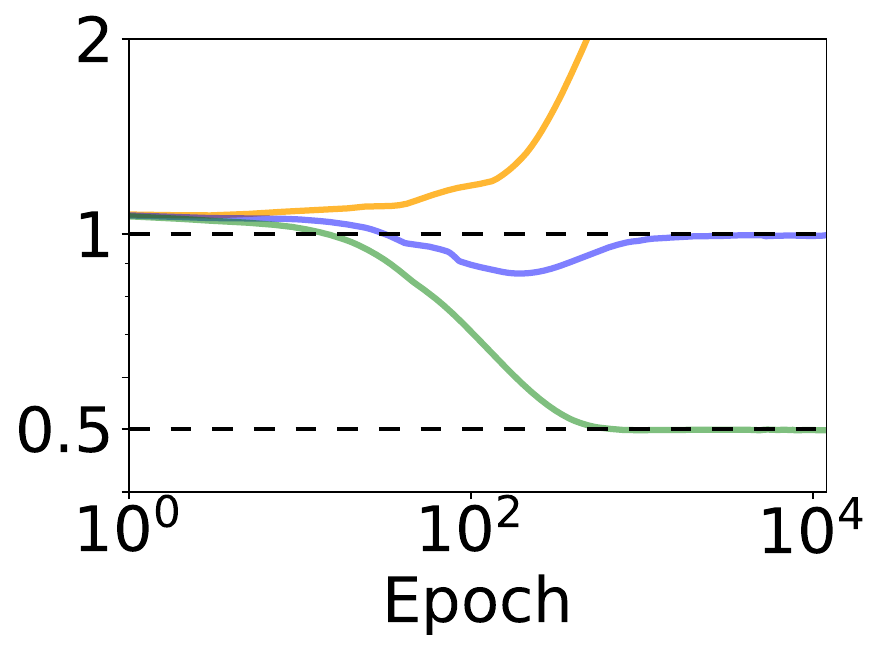}
        \caption{$\beta_1=\beta_2=0.99$}
        % \label{fig:success_2}
    \end{subfigure}
    \begin{subfigure}[b]{0.24\textwidth}
        \includegraphics[width=\textwidth]{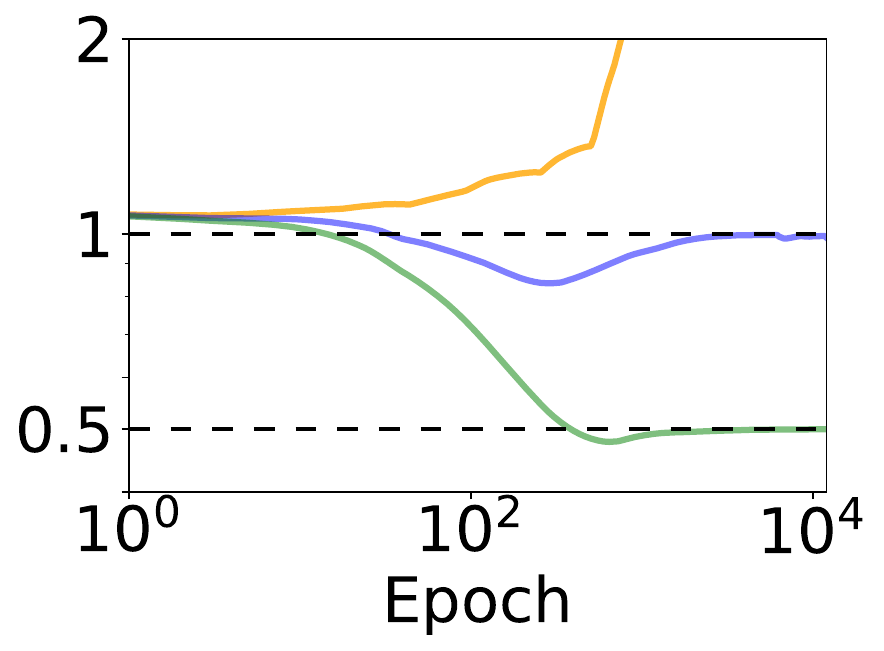}
        \caption{$\beta_1=\beta_2=0.999$}
        % \label{fig:success_3}
    \end{subfigure}
    \begin{subfigure}[b]{0.24\textwidth}
        \includegraphics[width=\textwidth]{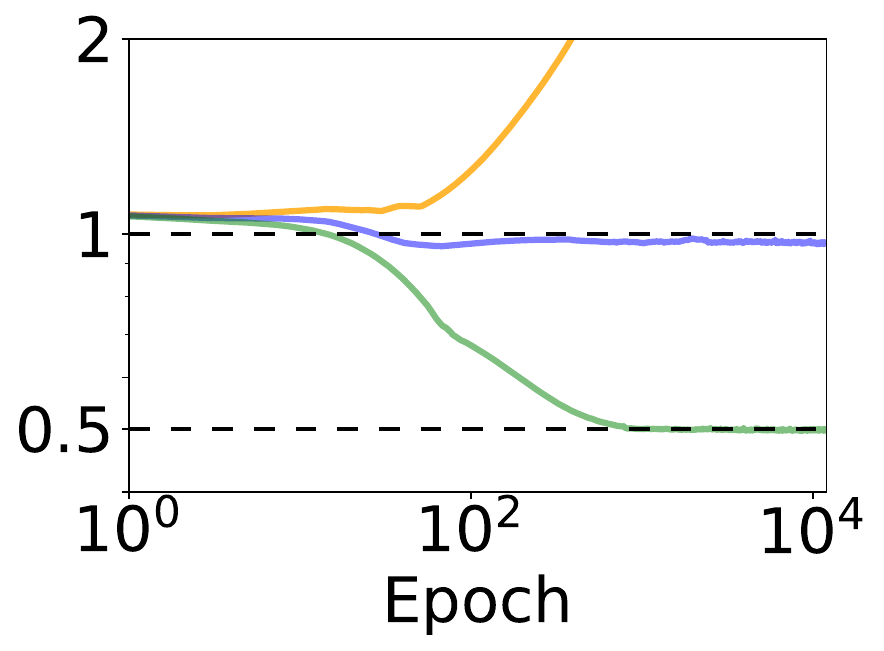}
        \caption{$\beta_1=0.9$, $\beta_2=0.95$}
        % \label{fig:success_1}
    \end{subfigure}
    \begin{subfigure}[b]{0.24\textwidth}
        \includegraphics[width=\textwidth]{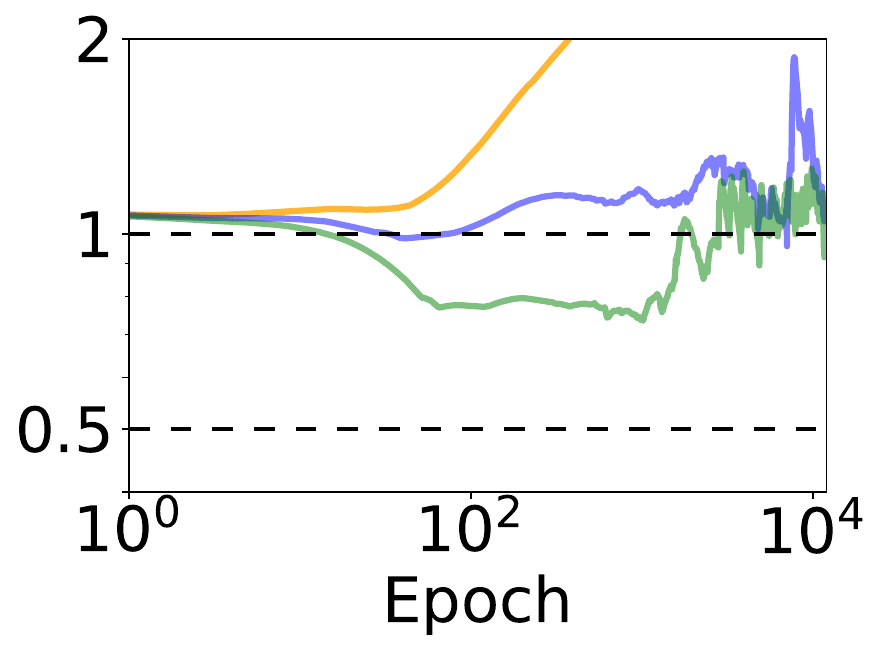}
        \caption{$\beta_1=0.9, \beta_2=0.999$}
        % \label{fig:failure}
    \end{subfigure}

    \begin{subfigure}[b]{0.6\textwidth}
        \includegraphics[width=\textwidth]{figures/legend.pdf}
        \phantomcaption
    \end{subfigure}
    \caption{$\ell_\infty$ norm of parameters for \adam and \adamw with different $\beta_1, \beta_2$ for seed $1$}
    \label{fig:ptb_seed_1}
\end{figure}

\begin{figure}[h]
    \centering
    \begin{subfigure}[b]{0.24\textwidth}
        \includegraphics[width=\textwidth]{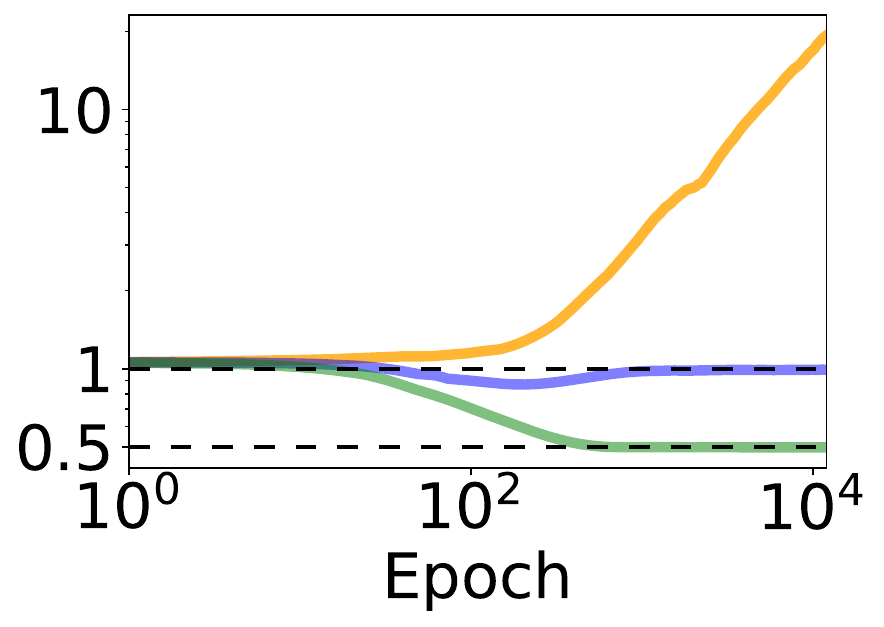}
        \caption{$\beta_1=\beta_2=0.99$}
        % \label{fig:success_2}
    \end{subfigure}
    \begin{subfigure}[b]{0.24\textwidth}
        \includegraphics[width=\textwidth]{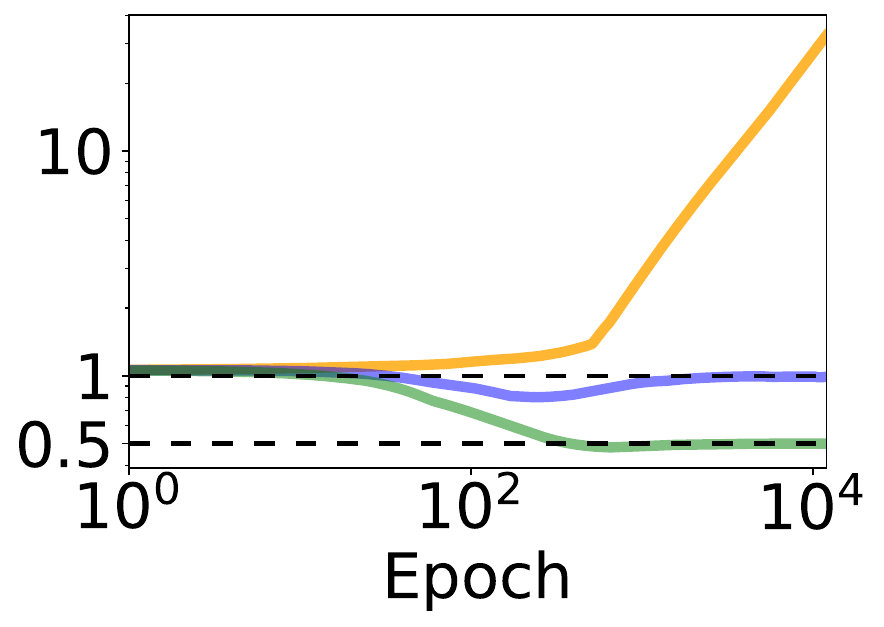}
        \caption{$\beta_1=\beta_2=0.999$}
        % \label{fig:success_3}
    \end{subfigure}
    \begin{subfigure}[b]{0.24\textwidth}
        \includegraphics[width=\textwidth]{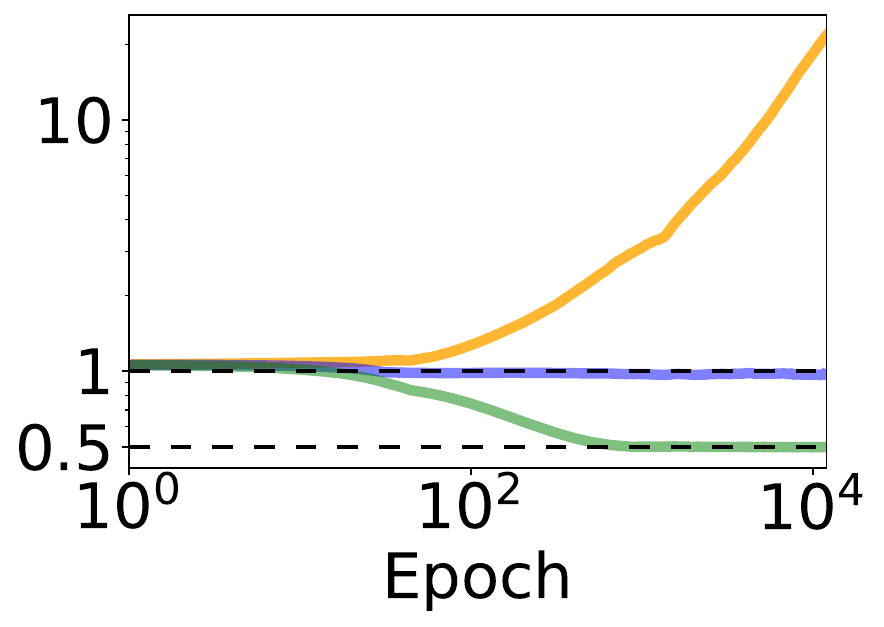}
        \caption{$\beta_1=0.9$, $\beta_2=0.95$}
        % \label{fig:success_1}
    \end{subfigure}
    \begin{subfigure}[b]{0.24\textwidth}
        \includegraphics[width=\textwidth]{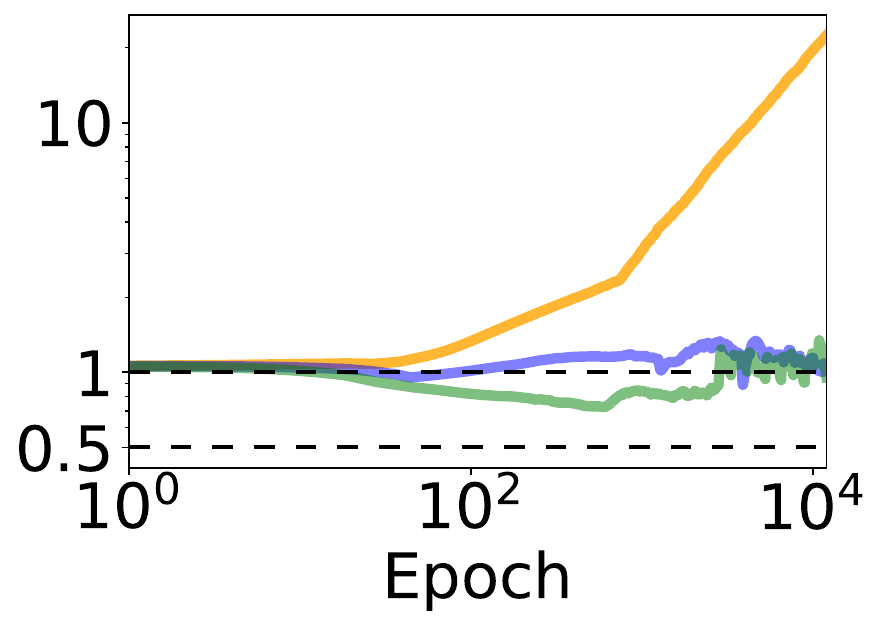}
        \caption{$\beta_1=0.9, \beta_2=0.999$}
        % \label{fig:failure}
    \end{subfigure}

    \begin{subfigure}[b]{0.6\textwidth}
        \includegraphics[width=\textwidth]{figures/legend.pdf}
        \phantomcaption
    \end{subfigure}
    \caption{$\ell_\infty$ norm of parameters for \adam and \adamw with different $\beta_1, \beta_2$ for seed $0$. The range of y-axis is extended to show the full result of \adam. }
    \label{fig:ptb_seed_0_full}
\end{figure}

\begin{figure}[h]
    \centering
    \begin{subfigure}[b]{0.24\textwidth}
        \includegraphics[width=\textwidth]{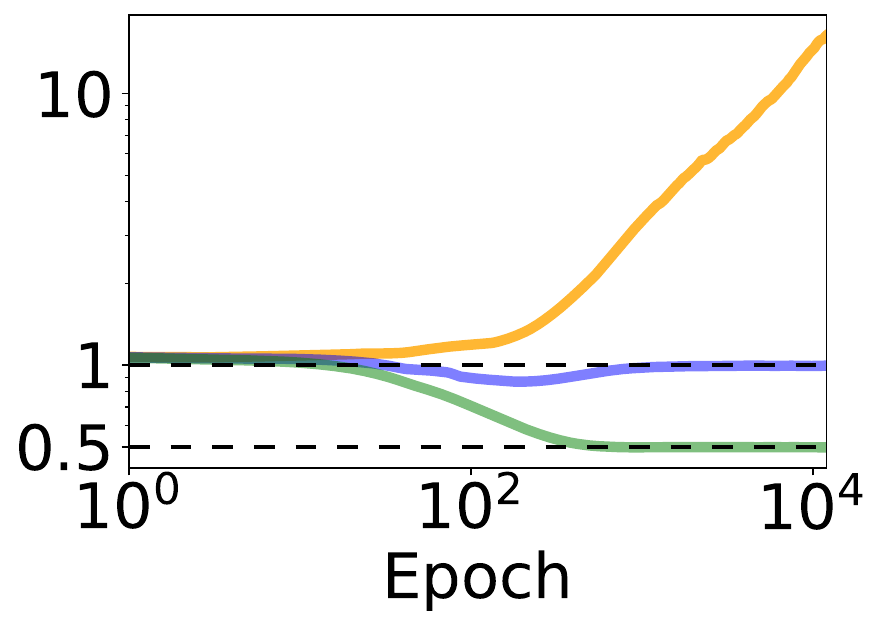}
        \caption{$\beta_1=\beta_2=0.99$}
        % \label{fig:success_2}
    \end{subfigure}
    \begin{subfigure}[b]{0.24\textwidth}
        \includegraphics[width=\textwidth]{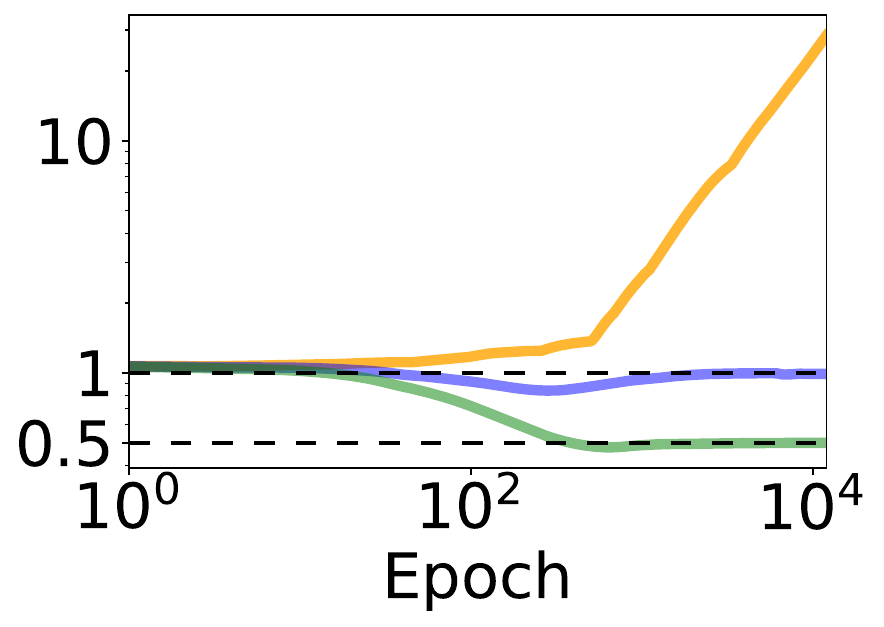}
        \caption{$\beta_1=\beta_2=0.999$}
        % \label{fig:success_3}
    \end{subfigure}
    \begin{subfigure}[b]{0.24\textwidth}
        \includegraphics[width=\textwidth]{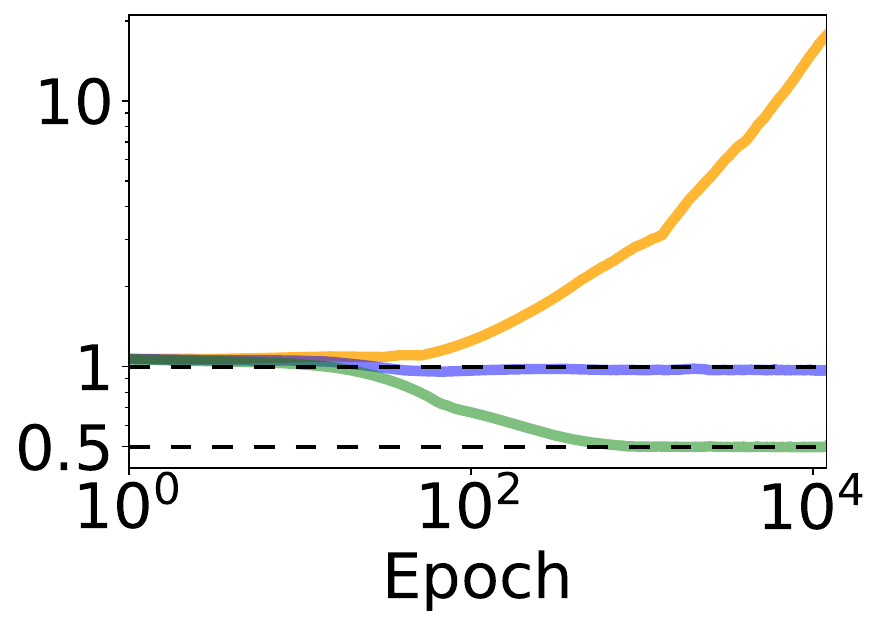}
        \caption{$\beta_1=0.9$, $\beta_2=0.95$}
        % \label{fig:success_1}
    \end{subfigure}
    \begin{subfigure}[b]{0.24\textwidth}
        \includegraphics[width=\textwidth]{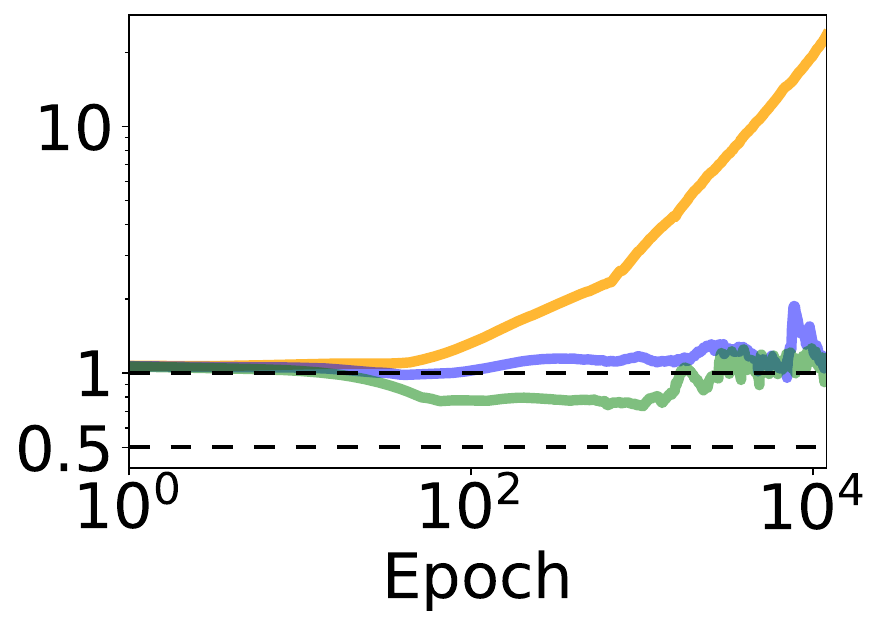}
        \caption{$\beta_1=0.9, \beta_2=0.999$}
        % \label{fig:failure}
    \end{subfigure}

    \begin{subfigure}[b]{0.6\textwidth}
        \includegraphics[width=\textwidth]{figures/legend.pdf}
        \phantomcaption
    \end{subfigure}
    \caption{$\ell_\infty$ norm of parameters for \adam and \adamw with different $\beta_1, \beta_2$ for seed $1$. The range of y-axis is extended to show the full result of \adam. }
    \label{fig:ptb_seed_1_full}
\end{figure}

\end{document}